\documentclass[]{fairmeta}

\usepackage{amsmath,amsfonts,bm}

\def\eqref#1{equation~\ref{#1}}

\def\1{\bm{1}}

\DeclareMathAlphabet{\mathsfit}{\encodingdefault}{\sfdefault}{m}{sl}
\SetMathAlphabet{\mathsfit}{bold}{\encodingdefault}{\sfdefault}{bx}{n}

\newcommand{\E}{\mathbb{E}}

\usepackage{booktabs}
\usepackage{hyperref}
\usepackage{url}
\usepackage{multirow}    
\usepackage{pifont}      
\usepackage{tablefootnote} 
\usepackage{enumitem}
\usepackage{graphicx}
\usepackage{amssymb} 
\usepackage{makecell}
\usepackage{epigraph}
\usepackage{float}
\usepackage{colortbl}
\usepackage[most]{tcolorbox}
\usepackage{xcolor}
\usepackage{amssymb} 
\usepackage{MnSymbol}  
\usepackage{amsthm}
\usepackage{wrapfig}

\definecolor{gray94}{gray}{.94}
\definecolor{gray90}{gray}{.90}
\definecolor{basegray}{RGB}{245,245,245}
\definecolor{rlvrblue}{RGB}{230,245,255}

\definecolor{lightred}{RGB}{255,240,240}
\definecolor{lightgreen}{RGB}{240,255,240}
\definecolor{lightblue}{RGB}{240,245,255}
\definecolor{lightgray}{RGB}{245,245,245}
\definecolor{MyGreen}{rgb}{0.13, 0.55, 0.13}

\newtcolorbox{AIbox}[2][]{aibox,title=#2,#1}

\definecolor{rliableolive}{HTML}{BBCC33}
\definecolor{rliableblue}{HTML}{77AADD}
\definecolor{rliablered}{HTML}{EE8866}
\definecolor{SDEblue}{RGB}{28 58 88}
\definecolor{cc1}{rgb}{1.0, 0.44, 0.37}
\definecolor{cc2}{rgb}{0.0, 0.2, 0.6}
\definecolor{cc3}{RGB}{255, 191, 0}
\definecolor{cc4}{RGB}{0, 128, 128}

\definecolor{myblue}{RGB}{0,77,64} 
\definecolor{mybg}{RGB}{240,248,247} 

\newtcolorbox{takeawaybox}{
  enhanced,
  colback=mybg,
  colframe=myblue,
  coltitle=myblue,
  boxrule=0.6pt,
  arc=6pt,
  left=6pt,
  right=6pt,
  top=6pt,
  bottom=6pt,
  title=\textsc{Takeaway}
}

\newtheorem{theorem}{Theorem}[section]
\newtheorem{proposition}[theorem]{Proposition}
\newtheorem{lemma}[theorem]{Lemma}

\newtheorem{corollary}[theorem]{Corollary}
\newtheorem{definition}[theorem]{Definition}

\title{
The Path Not Taken: \\
RLVR Provably Learns Off the Principals
}

\author[1,2, *]{Hanqing Zhu}
\author[2]{Zhenyu Zhang}
\author[1]{Hanxian Huang}
\author[1]{DiJia Su}
\author[1]{Zechun Liu}
\author[1]{Jiawei Zhao}
\author[1]{\newline Igor Fedorov}
\author[1]{Hamed Pirsiavash}
\author[2]{Zhizhou Sha}
\author[1]{Jinwon Lee}
\author[2]{David Z. Pan} 
\author[2, \dagger]{\newline Zhangyang Wang}
\author[1, \dagger]{Yuandong Tian}
\author[1, \dagger]{Kai Sheng Tai}

\affiliation[1]{Meta AI}
\affiliation[2]{The University of Texas at Austin}
\contribution[*]{Work done during an internship at Meta AI.}
\contribution[\dagger]{Equal advisory contribution.}

\metadata[Author emails:]{\email{hqzhu@utexas.edu}}

\abstract{
Reinforcement Learning with Verifiable Rewards (RLVR) reliably improves large-language-model reasoning while apparently modifying only a small fraction of parameters. We revisit this paradox and show that sparsity is a surface artifact of a \textbf{model-conditioned optimization bias}: for a fixed pretrained model, updates consistently localize to model-preferred parameter regions, remain highly consistent across runs, and are largely invariant to datasets and RL recipes. We mechanistically explain these dynamics with a \textbf{Three-Gate Theory}:
Gate I (KL Anchor) imposes a KL-constrained update;
Gate II (Model Geometry) steers steps \emph{off principal directions} into low-curvature, spectrum-preserving subspaces;
Gate III (Precision) hides micro-updates in non-preferred regions, making the off-principal bias appear as sparsity.
We then validate this theory and, to our knowledge, provide the first parameter-level characterization of RLVR’s learning dynamics: \textbf{RLVR learns off-principal directions in weight space}, achieving gains via minimal spectral drift, reduced principal-subspace rotation, and off-principal update alignment.
In contrast, SFT targets principal weights, distorts the spectrum, and even lags RLVR.

Together, these results provide the first parameter-space account of RLVR’s training dynamics, revealing clear regularities in how parameters evolve.
Crucially, we show that \textbf{RL operates in a distinct optimization regime from SFT}, so directly adapting SFT-era parameter-efficient fine-tuning (PEFT) methods can be flawed, as evidenced by our case studies on advanced sparse fine-tuning and LoRA variants.
We hope this work charts a path toward a \textit{white-box} understanding of RLVR and the design of \textit{geometry-aware, RLVR-native} learning algorithms, rather than repurposed SFT-era heuristics.
}

\begin{document}\maketitle

\section{Introduction}

Large Reasoning Models (LRMs), such as OpenAI-o3~\citep{jaech2024openai} and DeepSeek-R1~\citep{guo2025deepseek},
have advanced the ability of large language models to solve complex mathematical and programming tasks. 
A key driver is large-scale Reinforcement Learning with Verifiable Rewards (RLVR), 
which uses simple, easy-to-verify rewards to incentivize complex, multi-step reasoning.

Yet, despite these advances, the mechanisms by which RL shapes model representations and behavior remain poorly understood. 
Given the substantial computational resources devoted to RL~\citep{grok2025}, especially relative to supervised fine-tuning (SFT), and the emergence of striking new behaviors, 
one might naturally assume that such progress arises from significant parameter changes. 
However, recent evidence points in the opposite direction: RL induces \emph{sparse} parameter updates, whereas SFT yields \emph{dense} ones~\citep{mukherjee2025reinforcement}. 
This counterintuitive finding reveals a paradox, \textit{a high-cost, high-gain process that relies on surprisingly minimal weight modification.}

\begin{figure}[h]
    \centering
    \includegraphics[width=0.95\linewidth]{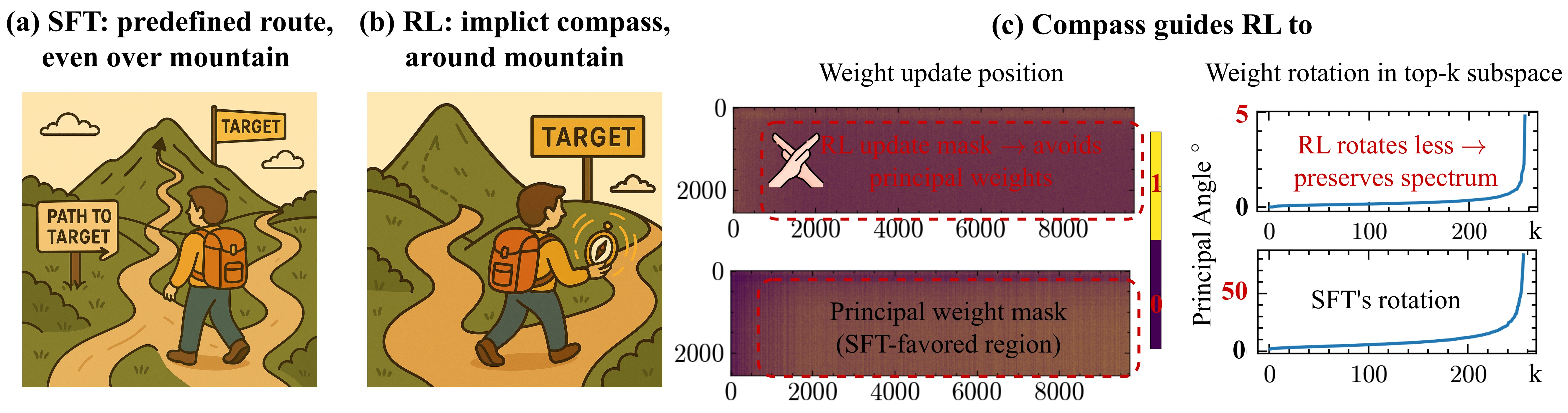}
    \vspace{-10pt}
    \caption{\textbf{SFT vs.\ RLVR: optimization geometry and evidence.}
    (a) \textbf{SFT} follows an externally guided route and traverses high-curvature directions (``over the mountain'') to reach the target.
    (b) \textbf{RLVR}, without an explicit teacher, behaves as if steered by an \emph{implicit compass} (a model-conditioned optimization bias), taking a low-curvature detour.
    (c) \textbf{Evidence.} \textit{Left:} positional maps comparing the \emph{update mask} (non-zero parameter updates) with the \emph{principal mask} (positions aligned with top-$k$ singular subspaces, defined by the largest-magnitude entries of the rank-$k$ SVD reconstruction~\cite{liu2025lift}; details in Sec.~\ref{sec:princ}). RLVR updates avoid principal-weight positions, whereas SFT targets them~\citep{meng2024pissa,liu2025lift}.
    \textit{Right:} principal-angle curves of the top-$k$ subspaces show that RLVR rotates less (spectrum preserved), while SFT rotates more.}
    \vspace{-10pt}
    \label{fig:teaser}
\end{figure}

\paragraph{\textbf{Key observation.}}
We resolve this paradox by uncovering a deeper mechanism behind the apparent sparsity: a \textbf{persistent, model-conditioned optimization bias}.
For a fixed pretrained model, this bias concentrates visible updates into a narrow, stable subset of parameters and remains strikingly invariant across diverse algorithms and datasets—\emph{a model-conditioned feature}.
\texttt{bfloat16} precision further accentuates the apparent sparsity by attenuating micro-updates in non-preferred regions.
As illustrated in Fig.~\ref{fig:teaser}, we depict this bias as an \emph{implicit compass}: unlike SFT with an explicit teacher, RLVR is subtly guided during optimization even without one.

\paragraph{\textbf{Research Question.}}
These phenomena raise a central question about RL’s learning dynamics:

\begin{tcolorbox}[colback=pink!10!white,colframe=black,boxrule=0.9pt,boxsep=2pt,top=3pt,bottom=3pt,left=3pt,right=3pt]
\begin{center}
\textbf{\emph{Where does this optimization bias originate, and how does it shape parameter evolution during training?}}
\end{center}
\end{tcolorbox}

\paragraph{\textbf{Mechanistic explanation.}}
We formalize the mechanism behind RLVR’s optimization dynamics through a \textbf{Three-Gate Theory}.
Gate I (KL Anchor) enforces a KL-constrained update at each on-policy step.
Gate II (Model Geometry) then steers this update \textit{off the principal directions} toward low-curvature, spectrum-preserving subspaces embedded in the structured optimization landscape of a \emph{pretrained} model, unlike training from a \emph{randomly initialized} model.
This geometry gate explains the \emph{model-conditioned} nature of the bias: it arises from the pretrained landscape rather than particular datasets or RL recipes.
Gate III (Precision) acts as a realization filter by hiding those micro-updates in non-preferred regions, making the off-principal bias \emph{appear} sparse.

\paragraph{\textbf{Experimental validation.}}
We validate this theory with a comprehensive suite of experiments, uncovering \textbf{striking optimization dynamics}: \emph{RLVR learns off the principal directions}, operating in a regime \emph{disjoint} from SFT’s. We show that
(i) \emph{RLVR preserves} the pretrained spectral structure with , whereas \emph{SFT distorts} it;
(ii) \emph{RLVR avoids} principal weights—the high-energy directions indicated by rank-$k$ SVD reconstructions—whereas parameter-efficient \emph{SFT} \emph{targets} them~\citep{liu2025lift}; and
(iii) \emph{RLVR depends on} the pretrained geometry: function-preserving orthogonal rotations \emph{abolish} the effect of update locality overlap, consistent with a \emph{model-conditioned} optimization bias.

\paragraph{\textbf{Rethinking learning algorithms for RLVR.}}
Beyond characterizing learning dynamics, our findings reveal that RLVR operates in a regime fundamentally distinct from SFT.
Consequently, direct carry-over of SFT-era parameter-efficient fine-tuning (PEFT) methods can be flawed, especially those over-aligned with SFT’s optimization geometry.
\emph{(1) Sparse fine-tuning.}
Restricting updates to \emph{principal} weights, an SFT prior~\citep{liu2025lift}, yields the weakest optimization trajectory and markedly degrades performance, as reflected by both forward-KL drift~\citep{shenfeld2025rl} and accuracy.
Conversely, updating \emph{non-principal, low-magnitude} weights, precisely the off-principal regime predicted by our theory, closely tracks the dense RLVR trajectory, validating our parameter-level understanding.
\emph{(2) LoRA variants.}
A recent report~\citep{schulman2025lora} finds that low-rank LoRA (even rank-1) can match full-parameter performance in RL.
However, our analysis challenges their belief that advanced LoRA variants such as PiSSA~\citep{meng2024pissa} bring further gains: PiSSA explicitly targets principal weights-suitable for SFT but \emph{misaligned} with RLVR’s off-principal dynamics.
Empirically, PiSSA offers no obvious gain over standard LoRA.
Moreover, enforcing principal-direction updates-e.g., via learning-rate scaling, a common requirement to match full-parameter performance with low-rank adapters, often \emph{destabilizes} training and precipitates early collapse.

\paragraph{\textbf{Contributions.}}
Our work makes the following key contributions:
\begin{itemize}[leftmargin=*]
  \item \textbf{Observation (Sec.~\ref{sec:obs-bias}).} We identify a \textbf{persistent, model-conditioned optimization bias} in RLVR fine-tuning that is largely invariant to datasets and RL variants, yet highly consistent for a fixed pretrained model.
  \item \textbf{Theory (Sec.~\ref{sec:theory}).} We propose the \textbf{Three-Gate Theory}—\emph{KL Anchor, Model Geometry, and Precision}—which mechanistically explains how RL updates are constrained, steered, and filtered to produce the observed optimization pattern.
  \item \textbf{Evidence (Sec.~\ref{sec:gate2-validation}).} We provide \textbf{a parameter-level validation} contrasting the training dynamics of RL and SFT, including reduced spectral drift, smaller principal-subspace rotation, low overlap with principal weights, and basis-rotation interventions that isolate geometry as the steering core.
  \item \textbf{Insight (Sec.~\ref{sec:peft-implications}).} We show that SFT-era sparse and low-rank priors (e.g., principal-targeted variants) are misaligned with RLVR’s off-principal dynamics, motivating geometry-aware, RLVR-native learning algorithms.
\end{itemize}

Together, these findings provide the \textbf{first parameter-space account} linking RL optimization dynamics to weight evolution, complementing concurrent work that focuses primarily on policy-level or distributional effects~\citep{wu2025invisible, shenfeld2025rl}.
\textbf{Crucially}, RLVR operates in a distinct, geometry-driven optimization regime from SFT, calling for the development of \emph{RL-native, geometry-aware} PEFT methods (see Sec.~\ref{sec:peft-implications}) and marking a step toward a white-box understanding of RLVR training.

\section{A Persistent, Model-Conditioned Optimization Bias in RLVR}
\label{sec:obs-bias}

We begin with the observation that RL induces sparse parameter updates, and go beyond quantification to ask \emph{where RL localizes these changes} in order to understand the underlying mechanism.
Our analysis reveals a \textbf{model-conditioned optimization bias}: for a fixed pretrained model, RL consistently routes visible updates to specific regions of the network, highly consistent across runs and largely invariant to datasets and RL variants.
We further find that the observed sparsity is a \emph{superficial readout} of this bias, amplified by \texttt{bfloat16} precision, which attenuates micro-updates in non-preferred regions.

\paragraph{\textbf{Model suite.}}
We analyze publicly released checkpoints, as shown in Tab.~\ref{tab:sparsity_joint}. 
The suite spans multiple RLVR variants (e.g., GRPO, DAPO, Reinforcement++), diverse data domains, and several model families and types (dense and Mixture-of-Experts).
We place particular emphasis on \texttt{DeepSeek-R1-Distill-Qwen-1.5B} (\texttt{DS-Qwen-1.5B}), for which a long-horizon RL checkpoint is available~\citep{liu2025prorl}.
This model serves as a robust case study given its extensive training for over 3,000 steps on a diverse data mixture encompassing mathematics, coding, STEM, logic puzzles, and instruction-following tasks.

\subsection{A Robust, \texttt{bfloat16}-aware Analysis of Update Sparsity}


\definecolor{SFTrow}{HTML}{F3F3F3}   
\definecolor{RLrow}{HTML}{EAF2FF}    

\begin{table}[t]
\centering
\caption{\small
\textbf{Update sparsity in SFT vs.\ RLVR.}
\textit{Higher} $\mathrm{sparsity}_{\mathrm{bf16}}$ indicates more weights unchanged.
RLVR is consistently much sparser than SFT.
$\dagger$ \textit{Mixed} denotes a diverse data source combining math, coding, STEM, logic puzzles, and instruction-following~\cite{liu2025prorl}.
}

\label{tab:sparsity_joint}
\setlength{\tabcolsep}{4pt}
\resizebox{0.8\textwidth}{!}{%
\vspace{-10pt}
\begin{tabular}{@{} l l l l r @{}}
\toprule
\textbf{Base Model} & \textbf{Finetuned (FT) Model} & \textbf{Algorithm} & \textbf{Data} & \textbf{$\mathrm{sparsity}_{\mathrm{bf16}}$} \\
\midrule
\rowcolor{SFTrow}
\href{https://huggingface.co/Qwen/Qwen2.5-Math-1.5B}{Qwen-1.5B} &
\href{https://huggingface.co/deepseek-ai/DeepSeek-R1-Distill-Qwen-1.5B}{DS-R1-Distill-Qwen-1.5B} &
SFT & Mixed & 2.8\% \\
\rowcolor{RLrow}
\href{https://huggingface.co/deepseek-ai/DeepSeek-R1-Distill-Qwen-1.5B}{DS-R1-Distill-Qwen-1.5B} &
\href{https://huggingface.co/agentica-org/DeepScaleR-1.5B-Preview}{DeepScaleR-1.5B-Preview} &
GRPO & Math & 53.8\% \\
\rowcolor{RLrow}
\href{https://huggingface.co/deepseek-ai/DeepSeek-R1-Distill-Qwen-1.5B}{DS-R1-Distill-Qwen-1.5B} &
\href{https://huggingface.co/agentica-org/DeepCoder-1.5B-Preview}{DeepCoder-1.5B-Preview} &
GRPO & Code & 45.5\% \\
\rowcolor{RLrow}
\href{https://huggingface.co/deepseek-ai/DeepSeek-R1-Distill-Qwen-1.5B}{DS-R1-Distill-Qwen-1.5B} &
\href{https://huggingface.co/Fate-Zero/Archer-Code-1.5B}{Archer-Code-1.5B} &
GRPO & Code & 52.5\% \\
\rowcolor{RLrow}
\href{https://huggingface.co/deepseek-ai/DeepSeek-R1-Distill-Qwen-1.5B}{DS-R1-Distill-Qwen-1.5B} &
\href{https://huggingface.co/nvidia/Nemotron-Research-Reasoning-Qwen-1.5B}{NV-ProRL} &
GRPO & Mixed$\dagger$ & 38.4\% \\
\rowcolor{RLrow}
\href{https://huggingface.co/deepseek-ai/DeepSeek-R1-Distill-Qwen-1.5B}{DS-R1-Distill-Qwen-1.5B} &
\href{https://huggingface.co/nvidia/Nemotron-Research-Reasoning-Qwen-1.5B}{NV-ProRL-v2} &
Reinforcement++ & Mixed$\dagger$ & 36.3\% \\
\midrule
\rowcolor{SFTrow}
\href{https://huggingface.co/Qwen/Qwen3-8B-Base}{Qwen3-8B-Base} &
\href{https://huggingface.co/Kwai-Klear/Klear-Reasoner-8B-SFT}{Klear-Reasoner-8B-SFT} &
SFT & Math+Code & 0.6\% \\
\rowcolor{RLrow}
\href{https://huggingface.co/Kwai-Klear/Klear-Reasoner-8B-SFT}{Klear-Reasoner-8B-SFT} &
\href{https://huggingface.co/Kwai-Klear/Klear-Reasoner-8B}{Klear-Reasoner-8B} &
GRPO & Math+Code & 69.5\% \\
\rowcolor{RLrow}
\href{https://huggingface.co/Qwen/Qwen3-8B-Base}{Qwen3-8B-Base} &
\href{https://huggingface.co/TMLR-Group-HF/GT-Qwen3-8B-Base}{GT-Qwen3-8B-Base} &
GRPO & Math & 79.9\% \\
\rowcolor{RLrow}
\href{https://huggingface.co/Qwen/Qwen3-8B-Base}{Qwen3-8B-Base} &
OURS &
DAPO & Math & 79.7\% \\
\midrule
\rowcolor{SFTrow}
\href{https://huggingface.co/Qwen/Qwen3-14B-Base}{Qwen3-14B-Base} &
\href{https://huggingface.co/ReasoningTransferability/UniReason-Qwen3-14B-think-SFT}{UniReason-Qwen3-14B-think-SFT} &
SFT & Math & 18.8\% \\
\rowcolor{RLrow}
\href{https://huggingface.co/Qwen/Qwen3-14B-Base}{Qwen3-14B-Base} &
\href{https://huggingface.co/ReasoningTransferability/UniReason-Qwen3-14B-RL}{UniReason-Qwen3-14B-RL} &
GRPO & Math & 68.3\% \\
\midrule
\rowcolor{RLrow}
\href{https://huggingface.co/Qwen/Qwen3-4B}{Qwen3-4B} &
\href{https://huggingface.co/POLARIS-Project/Polaris-4B-Preview}{Polaris-4B-Preview} &
DAPO & Math & 79.3\% \\
\rowcolor{RLrow}
\href{https://huggingface.co/deepseek-ai/DeepSeek-R1-Distill-Qwen-7B}{DS-R1-Distill-Qwen-7B} &
\href{https://huggingface.co/POLARIS-Project/Polaris-7B-Preview}{Polaris-7B-Preview} &
DAPO & Math & 61.7\% \\
\rowcolor{RLrow}
\href{https://huggingface.co/Qwen/Qwen3-30B-A3B}{Qwen3-30B-A3B} &
\href{https://huggingface.co/forestliutc/UloRL}{UloRL-A3B} &
GRPO & Math & 91.7\% \\
\bottomrule
\end{tabular}}
\vspace{-10pt}
\end{table}

\textbf{A \texttt{bfloat16}-aware probe for unchanged weights.}
\texttt{bfloat16} (bf16) is standard in modern RL frameworks like verl~\citep{sheng2024hybridflow}, to improve throughput without compromising performance.
However, analyzing parameter changes under bf16 requires a careful probe.
Its unique numerical format, with only 7 mantissa bits for precision, means that the smallest representable difference between two numbers scales with their magnitude.
Consequently, a fixed absolute-tolerance check as used in~\citep{mukherjee2025reinforcement}, is \textit{unreliable}, which can over- or under-report the fraction of unchanged weights (see Appendix~\ref{app:bf16diff}). 

To ensure a rigorous report, 
we adopt a numerically robust, \texttt{bfloat16}-aware probe to define the update sparsity $\mathrm{sparsity}_{\mathrm{bf16}}$ as the fraction of parameters that remain unchanged.
\begin{definition}[Unchanged Weight in bf16]
\label{def:bf16-unchanged}
Let $w_i,\widehat w_i\in\mathbb{R}$ be scalars stored in bf16 (finite, nonzero).
We say $w_i$ is \emph{unchanged} with respect to $\widehat w_i$ iff
\begin{equation}
\small
\label{eq:bf16-close}
\bigl|\,\widehat{w}_i - w_i\,\bigr|
\;\le\;
\eta\,\max \bigl(|w_i|,\;|\widehat{w}_i|\bigr),
\qquad \eta=10^{-3}.
\end{equation}
Choosing $\eta{=}10^{-3}<2^{-9}$ makes \eqref{eq:bf16-close} equivalent to bitwise equality (See Appendix~\ref{app:bf16diff-therom},).
\end{definition}

\begin{definition}[bf16-aware Update Sparsity]
\label{def:bf16-sparsity}
Write $x\approx^{\mathrm{bf16}}_{\eta}y$ for Def.~\ref{def:bf16-unchanged}.
Define the bf16 change count
$\|\theta^1-\theta^0\|^{\mathrm{bf16}}_{0,\eta}
\;:=\;
\bigl|\{\,i:\ \theta^1_i\not\approx^{\mathrm{bf16}}_{\eta}\theta^0_i\,\}\bigr|
$ and the corresponding sparsity
\begin{equation}
\small
\mathrm{sparsity}_{\mathrm{bf16}}(\theta^0,\theta^1;\eta)
\;:=\;
1-\|\theta^1-\theta^0\|^{\mathrm{bf16}}_{0,\eta}/n.
\end{equation}
where $n$ is the total number of parameters.
Values near $1$ indicate few stored changes, while values near $0$ indicate dense apparent change.
\end{definition}

\textbf{RLVR update sparsity results.}
As shown in Tab.~\ref{tab:sparsity_joint}, our analysis confirms that RL yields substantially higher update sparsity than  SFT. 
Across models, SFT sparsity is consistently low (typically 0.6\%–18.8\%), whereas RL sparsity is an order of magnitude higher, ranging from 36\% to 92\%. 
However, absolute levels on recent checkpoints are lower than earlier reports~\citep{mukherjee2025reinforcement}, underscoring the need for bf16‑aware probes and re‑evaluation on current models.

\subsection{RLVR Exhibits Model-Conditioned Update Locality}

Magnitude alone does not reveal \emph{where} changes occur, impeding deep analysis of \emph{how} sparse changes arise.
\begin{wraptable}{r}{7.5cm}
\vspace{-10pt}
\caption{\small Cross-run stability for 13th block.}
\vspace{-10pt}
\label{tab:jaccard_layers}
\resizebox{0.98\linewidth}{!}{%
\begin{tabular}{@{} lcc @{}}
\toprule
\textbf{Layer} & \textbf{Jaccard Overlap} & \textbf{Random Baseline} \\
\midrule
Q & 0.580 & 0.430 \\
K & 0.580 & 0.413 \\
V & 0.597 & 0.467 \\
O & 0.552 & 0.373 \\
MLP-down & 0.585 & 0.453 \\
MLP-up & 0.578 & 0.443 \\
MLP-gate & 0.575 & 0.437 \\
\bottomrule
\end{tabular}
}
\vspace{-15pt}
\end{wraptable}
We therefore examine the \emph{updated subnetwork}.
We use 5 independent RLVR checkpoints from the same \texttt{DS-Qwen-1.5B} in Tab.~\ref{tab:sparsity_joint}, trained on different datasets and RLVR algorithms.
For each layer $\ell$ and run $r$, we first
form the bf16-aware \emph{changed} mask $M^{(r)}_\ell := \mathbf{1} \big[\,W^{(r)}_\ell \not\approx^{\mathrm{bf16}}_{\eta} W^0_\ell\,\big]$ (Def.~\ref{def:bf16-sparsity}) against the base weights $W^0_\ell$.

\textbf{Stability across runs (\emph{Is the bias persistent?})}~
We first analyze their spatial agreement using \textit{Jaccard Overlap}.
For runs $r,s$, let $A=\{(i,j):M^{(r)}_{\ell,ij}=1\}$ and $B=\{(i,j):M^{(s)}_{\ell,ij}=1\}$.
We report the mean off-diagonal of the pairwise Jaccard matrix $J(A,B)=\frac{|A\cap B|}{|A\cup B|}$ and compare it to the independent Bernoulli baseline
$\mathbb{E}[J]=\tfrac{pq}{p+q-pq}$.
As summarized in Tab.~\ref{tab:jaccard_layers}, Jaccard is consistently high across runs, confirming a shared footprint when trained from the same base model, with Jaccard matrix shown in Fig.~\ref{fig:jaccard}.

\begin{figure}[t]
    \centering
    \includegraphics[width=0.95\linewidth]{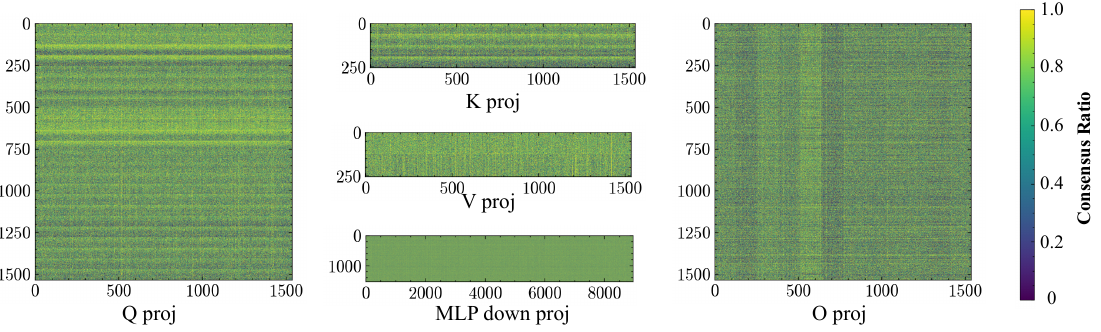}
    \vspace{-10pt}
    \caption{\textbf{Consensus ratio of weight updates.}
        Across five RLVR runs, we plot the 13th layer’s projections (Q/K/V/O) and the MLP down projection.
        Lighter bands mark coordinates updated in most runs, revealing a stable, stripe-like routing pattern rather than random scatter (zoom in to see fine structure).
    }
    \label{fig:strips}
    \vspace{-15pt}
\end{figure}

\textbf{Consensus ratio (Where do updates land?)}~
Stability alone does not indicate \emph{where} updates land. We therefore visualize and analyze the consensus ratio
$C_{\ell,ij} = \tfrac{1}{R}\sum_{r=1}^{R} M^{(r)}_{\ell,ij}$,
the fraction of runs realizing a \emph{weight update} at coordinate $(i,j)$.
Values near $1$ indicate that \emph{all} runs consistently change that weight; values near $0$ indicate
that none do.
As shown in Fig.~\ref{fig:strips}, consensus maps reveal contiguous row/column bands, stripe-like, localized routing rather than scattered noise.
Especially, there are obvious \emph{row-wise stripes in} Q/K/V projections and \emph{column-wise stripes in} O projections.
This exposes a clear \textbf{optimization bias}: 
\emph{RLVR consistently concentrates updates in specific regions of the parameter matrices for a fixed pretrained model,  
even though the five runs use disjoint data and RL variants.}

\textbf{Temporal stability (How does the bias emerge over time?)}
To examine \emph{within‑run} dynamics, we track the row‑wise ratio
$\rho_{\ell,i}(t)=\tfrac{1}{n_\ell}\sum_j M_{\ell,ij}(t)$
and column‑wise ratio
$\kappa_{\ell,j}(t)=\tfrac{1}{m_\ell}\sum_i M_{\ell,ij}(t)$
across checkpoints at $t$ steps. 
On \texttt{DS‑Qwen‑1.5B} (training setting in Appendix~\ref{appx:training}), the \emph{relative} profiles $\rho_{\ell,\cdot}(t)$
and $\kappa_{\ell,\cdot}(t)$ remain aligned while
\begin{wrapfigure}{r}{0.75\textwidth}
    \vspace{-10pt}
    \centering
    \includegraphics[width=\linewidth]{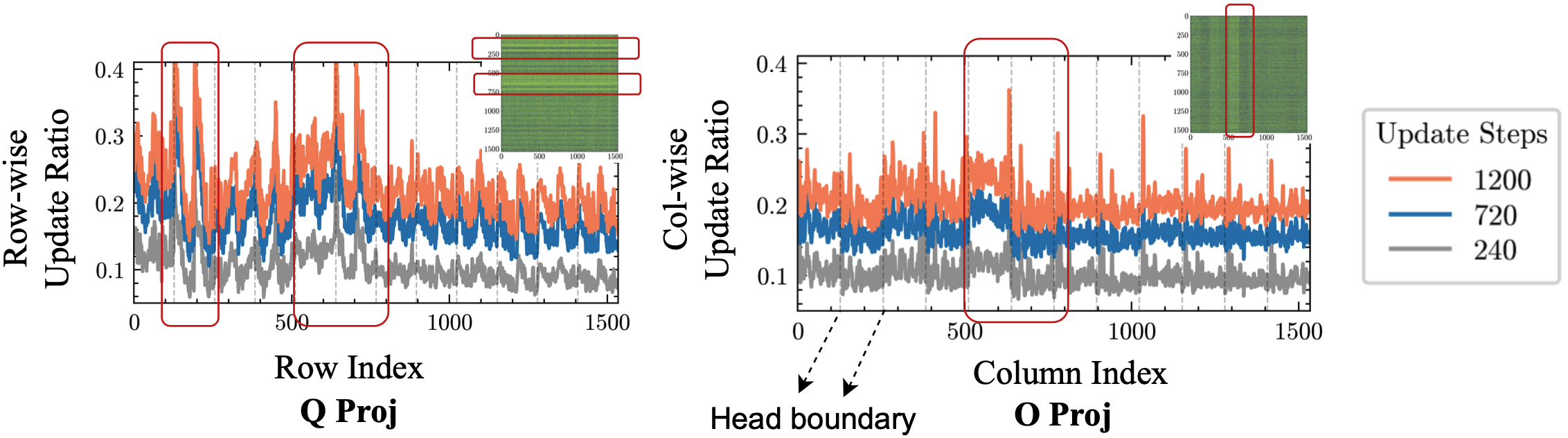}
    \vspace{-15pt}
    \caption{\textbf{Temporal emergence of the optimization bias} with row and column-wise update ratios for the 13th attention block across gradient update steps ($t\!\in\!\{240,720,1200\}$), smoothed with a 3-step window. The row-dominant (Q) and column-dominant (O) patterns are consistent with the bias structures in Fig.~\ref{fig:strips}.
    We visualize the head boundaries with grey dashed lines. The bias appears not only across heads but also within heads.
    }
    \label{fig:rowwise_dynamics}
    \vspace{-15pt}
\end{wrapfigure}
overall density grows as shown in Fig.~\ref{fig:rowwise_dynamics}: peaks and troughs persist. 
The routing bias \emph{emerges early} and is \emph{reinforced over training}, indicating a temporally stable phenomenon rather than a transient artifact.
Moreover, the peak is consistent with the bias structure shown in Fig.~\ref{fig:strips}.
We also show their remaining column-wise (Q) and row-wise (O) update ratio dynamics in Fig.~\ref{fig:rowwise_dynamics_comp}, without a clear trend, \textit{indicating the bias is indeed structured, not random.}

\textbf{Across model families:(\emph{Is the bias generic?})}
We observe similar stripe-structured footprints on \texttt{Llama} and \texttt{Mistral} (Fig.~\ref{fig:extra_llama} in Appendix), suggesting the routing bias is generic to RLVR.

\subsection{Sparsity Is a Superficial Artifact of the Optimization Bias}

The stable footprint of where updates land, persisting both throughout training and in the final model, suggests the focus should move from sparsity itself to the underlying optimization bias.

We find that sparsity is actually the \emph{readout} of this optimization bias, whose visibility is amplified by the precision limits of bf16 storage. Because bf16 has a limited mantissa, changes smaller than the unit-in-the-last-place (ULP) threshold (Lemma~\ref{lem:ULP-lens}) are not representable. Therefore, if RLVR consistently routes sub-ULP updates toward a particular subset of parameters, the stored values will not change, and the result appears as sparsity.

We test this hypothesis by increasing the learning rate to scale otherwise sub-ULP updates above the representable threshold. As predicted, the apparent update sparsity largely disappears. This directly challenges the interpretation of \citep{mukherjee2025reinforcement} that sparsity stems from zero gradients. 
Consistent with this view, concurrent work observes that sparsity mostly vanishes under \texttt{float32} storage~\citep{shenfeld2025rl} by increasing the precision, even though task performance does not improve.
Hence, our results point to sparsity as \textit{a byproduct of an optimization bias interacting with finite precision}. 

\textbf{Clarification on precision.}
\label{par:prec}
It may be tempting to blame precision limit for sparsity. In fact, \textsc{verl} keeps optimizer states and gradient reductions/accumulation in
\texttt{float32}\footnote{\href{https://github.com/volcengine/verl/blob/main/verl/workers/fsdp_workers.py\#L434-L441}{verl mixed-precision settings with \texttt{\{reduce\_type, buffer\_dtype\}=float32}}.}.
Thus, sparsity cannot be explained by precision alone.
It requires a consistent optimization bias during RL that concentrates visible changes in specific parameter regions throughout training.

\begin{tcolorbox}[colback=rliableblue!10!white,colframe=black,boxrule=0.9pt,boxsep=2pt,top=3pt,bottom=3pt,left=3pt,right=3pt]
\begin{center}
    \textbf{Aha Finding! \ \ RLVR exhibits a persistent, model-conditioned optimization bias in where updates land—highly consistent across runs and largely invariant to datasets and RL recipes.}
The observed sparsity is a superficial readout of this bias, amplified by bf16 precision.
\end{center}
\end{tcolorbox}

\section{A Mechanistic Theory of RL’s Unique Optimization Dynamics}
\label{sec:theory}

In the post-training era, RL has become a key stage,
albeit with intensive compute~\citep{grok2025}. 
Paradoxically (Sec.~\ref{sec:obs-bias}), these gains arise not from broad parameter changes but from selective, patterned edits that reveal a persistent optimization bias.
Understanding this distinctive training behavior raises the central question:
\begin{tcolorbox}[colback=pink!10!white,colframe=black,boxrule=0.9pt,boxsep=2pt,top=3pt,bottom=3pt,left=3pt,right=3pt]\begin{center}
\textbf{\emph{Where does this optimization bias originate, and how does it shape parameter evolution?}}  
\end{center}
\end{tcolorbox}

We characterize these optimization dynamics with the \emph{Three-Gate Theory}, KL Anchor, Model Geometry, and Precision, which mechanistically explains how on-policy RL updates are \emph{constrained} via Gate I (KL Anchor; Sec.~\ref{sec:gate1-kl}), \emph{steered} via Gate II (Model Geometry; Sec.~\ref{sec:gate2-model}), and \emph{filtered} via Gate III (Precision; Sec.~\ref{sec:gate3-precision}) into the observed update pattern.

\textbf{Notations.}
We consider a large language model with parameters \(\theta\), defining a conditional distribution \(\pi_\theta(y\mid x)\) over possible output token sequences \(y=(y_1,\dots,y_T) \in \mathcal{Y}\) given a prompt \(x\in\mathcal{X}\) from the space $\mathcal{X}$.
Each sequence $y$ is composed of tokens from a vocabulary \(\mathcal{V}\) of size \(N\).

\subsection{Gate I: On-Policy RL Imposes a One-Step KL Leash}
\label{sec:gate1-kl}

We first show that online policy gradient updates yield a per‑step \emph{policy} KL bound (an \textbf{\textit{anchoring}} effect), which in turn limits parameter movement during the RLVR update.

\textbf{RLVR objective.}
Various RLVR algorithms including PPO, GRPO, DAPO, and REINFORCE++,
learn a policy $\pi_\theta$ by optimizing variants of a KL‑regularized objective:
\begin{equation}
\label{eq:obj}
\small
\max_{\theta}\mathbb{E}_{\textcolor{teal}{y \sim {\pi_\theta}(\cdot \mid x)},x \sim \mathcal{X}} [ R(x, y) - \beta \mathrm{KL}( \textcolor{teal}{\pi_\theta}(\cdot \mid x) \,\big\|\, \textcolor{red}{\pi_{ref}}(\cdot \mid x)) ].
\end{equation}
where $\pi_{\mathrm{ref}}$ is a fixed reference policy and $\beta \ge 0$ controls the KL regularization
($\beta=0$ recovers the clip‑only variants such as DAPO).
Rewards $R(x,y)$ are \emph{verifiable} and (after normalization) \emph{bounded} (e.g., pass/fail or execution scores).
Moreover, the surrogate typically uses the token-wise importance ratio
$w_t=\frac{\pi_\theta(y_t \mid x,y_{<t})}{\pi_{\mathrm{old}}(y_t \mid x,y_{<t})}$ with clipping relative to $\pi_{\mathrm{old}}$.

\textbf{One‑step surrogate.}~~
With~\eqref{eq:obj}, a standard sequence‑level online policy‑gradient surrogate is
\begin{equation}
\small
\label{eq:g1-obj}
\mathcal{L}_{\mathrm{PG}}(\theta)
=
-\,\mathbb{E}_{x\sim\mathcal{X},\,y\sim\pi_\theta(\cdot\mid x)}
\big[A^{\perp}(x,y)\,\log\pi_\theta(y\mid x)\big],
\end{equation}
where $A^{\perp}$ is a (normalized) advantage estimate, optionally \emph{shaped} by a reference-KL log‑ratio term.
In practice, updates are performed over mini-batches, with a collected batch of data, not in a fully on-policy manner.
But the resulting error after a small step size $\Delta\theta$ is $O(\|\Delta\theta\|^2)$ (Lemma~\ref{lem:frozen-second-order}).

\textbf{Implicit KL leash.}
The KL leash emerges as policy gradient methods can be understood as a conservative projection, keeping new policy close to its starting point while reweighting it toward higher-reward outcomes, not pulling it toward a potentially distant external distribution like SFT:

\begin{proposition}[One‑step policy‑KL leash]
\label{prop:one-step-kl-main}
Let $q(\cdot\mid x)$ be a full‑support reference and let $\tilde q_\beta(\cdot\mid x)\propto q(\cdot\mid x)\exp(R/\beta)$ denote the soft‑regularized improvement oracle. 
Let $\theta^+$ be the parametric fit obtained by the $M$‑projection of $\tilde q_\beta$ onto the policy class,
$\theta^{+}\in\arg\min_\theta D_{\mathrm{KL}}(\tilde q_\beta\|\pi_\theta)$.
Then, for a sufficiently small one‑step update,
\begin{equation}
\label{eq:kl-leash-minimal}
\small
D_{\mathrm{KL}}\!\big(\pi_{\theta^{+}}\ \|\ \pi_{\theta}\big)
\ \le\ (1+o(1))\ D_{\mathrm{KL}}\!\big(\tilde q_\beta\ \|\ \pi_\theta\big),
\end{equation}
where the $o(1)$ term vanishes as $D_{\mathrm{KL}}(\tilde q_\beta\|\pi_\theta)\to 0$.
\end{proposition}
Notably, even when the explicit KL term is removed (e.g., in DAPO with $\beta=0$), the ratio clipping trick still imposes a KL bound $O(\varepsilon^2)$ in the small‑step regime (Appendix.~\ref{appx:proof_clip}), confirmed empirically with a bounded KL divergence change during a DAPO run (Fig.~\ref{fig:dapo_loss}).

\textbf{Weight update constraint.}
Now we show the KL leash puts a constraint on the weight update $\Delta W$

\begin{proposition}[Policy‑KL leash $\Rightarrow$ weight bound]
\label{prop:policy-to-weight-main}
Assume $\log\pi_\theta$ is $C^3$ and let $F(\theta)$ denote the Fisher information. 
If a one‑step update $\theta^+=\theta+\Delta$ satisfies $D_{\mathrm{KL}}(\pi_{\theta^+}\|\pi_\theta)\le K$
and, on the update subspace, $F(\theta)\succeq \mu I$ for some $\mu>0$, then for $K$ sufficiently small
\begin{equation}
\label{eq:kl-leash-weight}
\small
\|\Delta\|_{F(\theta)} \triangleq \sqrt{\Delta^\top F(\theta)\Delta} \ \le\ \sqrt{2K}\,(1+o(1)),
\qquad
\|\Delta\|_2 \ \le\ \sqrt{\tfrac{2K}{\mu}}\,(1+o(1)).
\end{equation}
Consequently, for any weight matrix block $W\subset\theta$, $\ \|\Delta W\|_F \le \sqrt{2K/\mu}\,(1+o(1))$.
\end{proposition}

See a detailed proof for Proposition~\ref{prop:one-step-kl-main} in Appendix~\ref{appx:one-step-kl-main} and Proposition~\ref{prop:policy-to-weight-main} in Appendix~\ref{appx:policy-to-weight-main}.

\begin{tcolorbox}[colback=orange!10!white,colframe=black,boxrule=0.9pt,boxsep=2pt,top=3pt,bottom=3pt,left=3pt,right=3pt]
\textit{\textbf{Take-away 1:  RL update imposes an implicit KL leash (anchor effect), ensuring that the per-step drift from the current policy is small.}} 
This aligns with recent work arguing that even the final policy is KL-proximal~\citep{wu2025invisible, shenfeld2025rl}.
Our focus, however, is to understand how this leash affects the weight change dynamics. 

\end{tcolorbox}

\subsection{Gate II: Model Geometry Determines \emph{Where} a KL-Bounded Step Goes}
\label{sec:gate2-model}

\textbf{From Gate I to \textit{location}.}
Gate~I supplies a one-step KL leash that bounds the move, but it does not specify \emph{where} the update lands.
We propose Gate II(Model Geometry), where we argue that, unlike a randomly initialized network, a well-pretrained model possesses a highly structured geometry, e.g., spectral statistics and high-curvature directions during optimization, that determines where a KL-constrained update goes.

\textbf{Layerwise norm bound from the KL leash.}
Let $W_0$ be a pretrained linear block, $W_+=W_0+\Delta W$ the post-step block, and let $S_W\succeq \mu_W I$ be a per-layer curvature proxy. If the per-layer KL budget satisfies
$\tfrac12\langle \!\operatorname{vec}\Delta W,  S_W\,\operatorname{vec}\Delta W\rangle\le \delta_W$,
then (Appendix~\ref{lem:kl-frob-op-schur})
\begin{equation}
\small
\|\Delta W\|_F \le \sqrt{\tfrac{2\delta_W}{\mu_W}},\qquad
\|\Delta W\|_2 \le \sqrt{\tfrac{2\delta_W}{\mu_W}}.    
\end{equation}

We then show that this conservative update yields three consequences, preserving the pretrained weight spectrum rather than destroying it based on weight perturbation theory~\citep{stewart1998perturbation}.

\textbf{Limited subspace rotation.}
First, as shown in Theorem~\ref{thm:kl-only-angles}, the angle between the original and updated subspaces is quadratically bounded, meaning the fundamental directions are preserved.

\begin{theorem}[Constrained subspace rotation with Wedin’s sin--$\Theta$ theorem~\citep{wedin1972perturbation}]
\label{thm:kl-only-angles}
Let $\gamma_k := \sigma_k(W_0) - \sigma_{k+1}(W_0)$ be the singular value gap.
For any $k$ with $\gamma_k>0$,
\begin{equation}
\small
    \max(\big\|\sin\Theta(U_k(W_0),U_k(W_+))\big\|_2,\
\big\|\sin\Theta(V_k(W_0),V_k(W_+))\big\|_2
\ \le\ \frac{\|\Delta W\|_2}{\gamma_k}
\ \le\ \frac{\sqrt{2\delta_W/\mu_W}}{\gamma_k}.
\end{equation}
\end{theorem}

\textbf{Singular value stability.}
Second, the magnitudes of the principal components themselves are preserved. The change in each singular value is bounded by the norm of the update.

\begin{corollary}[Singular-value stability]
\label{cor:kl-only-svals}
For each $k$,
\begin{equation}
\small
    |\sigma_k(W_+)-\sigma_k(W_0)| \ \le\ \|\Delta W\|_2 \ \le\ \sqrt{\frac{2\delta_W}{\mu_W}},
\qquad
\sum_i \big(\sigma_i(W_+)-\sigma_i(W_0)\big)^2 \ \le\ \|\Delta W\|_F^2 \ \le\ \frac{2\delta_W}{\mu_W}.
\end{equation}
\end{corollary}

\textbf{Top-k energy preservation.}
Finally, these effects combine to ensure the cumulative energy of the top-k components of the weights remains stable.

\begin{corollary}[Top-$k$ energy and Ky Fan norms]
\label{cor:kyfan}
Let $\|\cdot\|_{(k)} := \sum_{i=1}^k \sigma_i(\cdot)$ be the Ky Fan $k$-norm. Then
\begin{equation}
\small
    \big|\,\|W_+\|_{(k)} - \|W_0\|_{(k)}\,\big|
\ \le\ \sum_{i=1}^k \big|\sigma_i(W_+)-\sigma_i(W_0)\big|
\ \le\ k\,\|\Delta W\|_2
\ \le\ k\,\sqrt{\tfrac{2\delta_W}{\mu_W}}.
\end{equation}
\end{corollary}

See the detailed proofs in Appendix~\ref{appx:gate2-lemmas} for all results presented here.

\begin{tcolorbox}[colback=orange!10!white,colframe=black,boxrule=0.9pt,boxsep=2pt,top=3pt,bottom=3pt,left=3pt,right=3pt]
\textbf{\textit{Take-away 2: Under the KL leash, RL updates tend to preserve the model's original weight structure rather than destroy it. 
This naturally favors updates in low-curvature directions of the optimization landscape, which avoids dramatic changes in model behavior.}}
Since directly quantifying curvature in LRM with long CoTs is computationally prohibitive, we instead adopt a powerful and efficient proxy, principal weights~\citep{liu2025lift}, as detailed in Sec.~\ref{sec:princ}.
\end{tcolorbox}

\subsection{Gate III: Precision Acts as a Lens Revealing the Compass}
\label{sec:gate3-precision}

Building on the optimization bias,  the bfloat16 with limited precision acts as a \textit{lens}: 
it hides those micro-updates that occur where the RL consistently holds a weak willingness to apply large changes.

\begin{corollary}[Magnitude-dependent realization threshold]
\label{cor:magnitude-threshold}
A stored weight $W_{ij}$ changes at a step iff \(|\Delta W_{ij}|\gtrsim\tfrac12\,\mathrm{ULP}_{\mathrm{bf16}}(W_{ij})\).
\end{corollary}

The effect of this gate has been discussed aforementioned.
We would emphasize again that precision is more an \emph{amplifier} for visible sparsity, not the \emph{cause} of optimization bias, as optimizer states, etc., are still in float32 (See Sec.~\ref{par:prec}).

\section{Theory-Guided Validation of RLVR’s Optimization Dynamics}
\label{sec:gate2-validation}

\begin{figure}[t]
    \centering
    \includegraphics[width=0.95\linewidth]{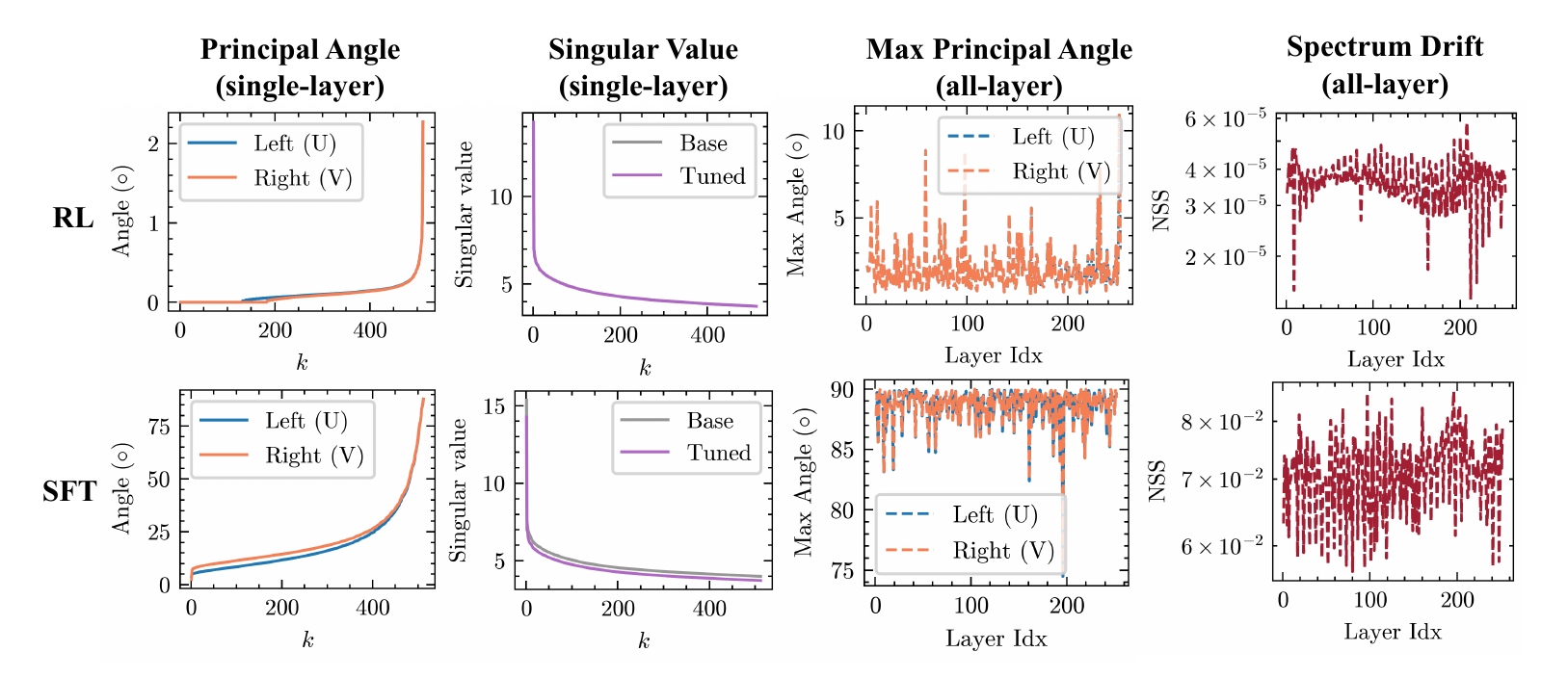}
    \vspace{-15pt}
    \caption{\small \textbf{Spectral geometry under SFT vs.\ RLVR on Qwen3-8B}~\citep{su2025klear}.
    Left: for an exemplar layer, top-$k$ principal angles and singular-value curves. 
    Right: across all layers, maximum principal angle and normalized spectral drift. 
    RLVR maintains a stable top-$k$ spectrum with minimal subspace rotation, unlike SFT. 
    See \texttt{DS-Qwen-1.5B} in Fig.~\ref{fig:rl-spec-1.5b} and Qwen3-14B-Base in Fig.~\ref{fig:rl-spec-14b}.}
    \label{fig:rl-spec}
    \vspace{-10pt}
\end{figure}

We conduct theory-guided experiments analyzing how RLVR modifies parameters and interacts with pretrained geometry. 
These results validate our central prediction: the pretrained model geometry steers \emph{KL-constrained} updates, yielding \emph{distinct, off-principal optimization dynamics} that set RLVR apart from SFT.

\subsection{RLVR Preserves Spectral Geometry, While SFT Distorts It}
\label{subsec:spec-rl-sft}
We begin by probing spectral changes to test whether RL updates are steered toward low-curvature, spectrum-preserving directions.
If so, RLVR should largely preserve the pretrained spectral structure, whereas SFT, lacking this steering, should significantly distort it.

\textbf{Setups.}
We analyze checkpoints from a standard SFT$\!\rightarrow$RLVR pipeline on Qwen3-8B-Base~\citep{su2025klear} and a long-horizon RL run on \texttt{DS-Qwen-1.5B}~\citep{liu2025prorl}. 
We also consider a setting where SFT and RL are applied separately to Qwen3-14B-Base, matched on in-domain math performance~\citep{huan2025does}.
In all cases, we compare base weights $W_0$ and fine-tuned weights $W_{+}$.

\textbf{Metrics.}
We compare the base weights $W_0$ with the finetuned weights $W_{+}$:
\begin{itemize}[leftmargin=1em,itemsep=-2pt, topsep=0pt]
\item \textbf{Subspace rotation.} For the top‑$k$ left ($U$)/right($V$) singular subspaces, we check the rotation using \textbf{principal angles} via
$\cos\theta_i(U):=\sigma_i\!\left(U_{0,k}^\top U_{+,k}\right)$ and 
$\cos\theta_i(V):=\sigma_i\!\left(V_{0,k}^\top V_{+,k}\right)$.
\item \textbf{Spectrum drift.} 
Beyond showing the singular value curve, we quantify singular‑value change with a normalized $\ell_2$ shift:
$\mathrm{NSS}(W) = \|\sigma(W_+)-\sigma(W_0)\|_2/\|\sigma(W_0)\|_2$
\end{itemize}

\textbf{Our findings.}
RLVR checkpoints exhibit a \emph{Insightably stable} spectrum within the top principal components: across layers, RLVR shows \emph{consistently small} principal-subspace rotation and \emph{minimal} spectral drift.
The singular-value profiles are \emph{even nearly identical} to the base model.
By contrast, SFT induces \emph{substantially larger} rotations and \emph{pronounced} drifts on the same metrics (Fig.~\ref{fig:rl-spec}).

\begin{figure}[t]
    \centering
    \includegraphics[width=0.92\linewidth]{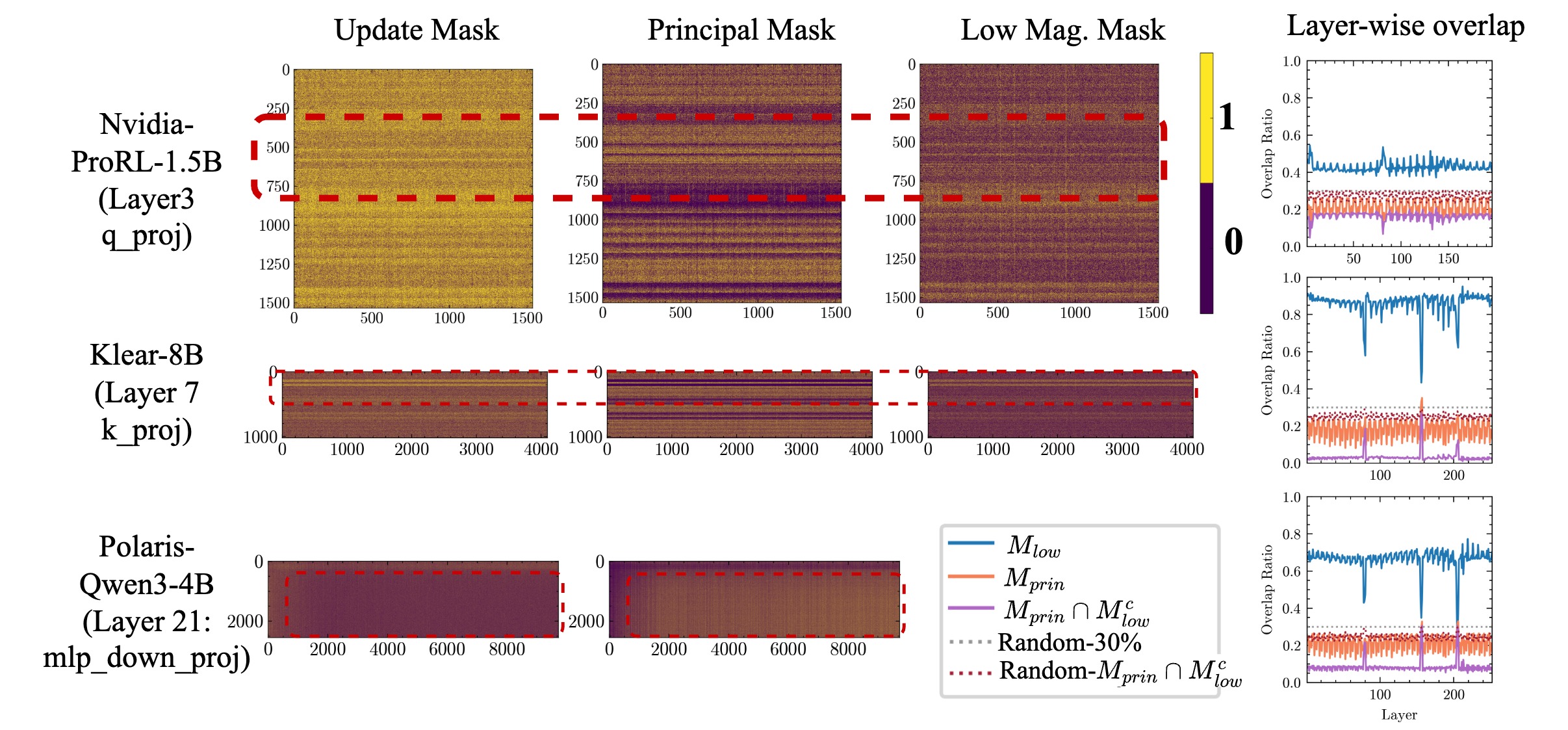}
    \vspace{-15pt}
    \caption{\small \textbf{RL avoids updating principal weights.} 
    We compare the RL update mask with principal weight mask $M_{princ}$, low magnitude mask $M_{low}$, and the one $M_{princ} \cap M_{low}^c$.
    The layer-wise overlap between RL updates and principal weights is consistently \emph{sub-random}, an effect more pronounced when removing its overlapped weights with $M_{low}$, i.e., $ M_{princ} \cap M_{low}^c$. 
    } 
    \label{fig:rl-avoids-principal}
    \vspace{-15pt}
\end{figure}

\subsection{RLVR Avoids Principal Weights, While SFT Targets Them}
\label{sec:princ}

We now move from macro-level spectral analysis to a micro-level examination of individual weights, probing \textit{which parameters RLVR favors or avoids to update}, a deeper investigation into the parameter-space dynamics.

\textbf{Principal weights as a proxy for high-curvature directions.}
Directly identifying high-curvature directions is computationally prohibitive, especially given LRM with long CoTs.
Instead, we adopt a powerful proxy from recent work~\cite{liu2025lift}, \textbf{\textit{principal weights},} which is defined as \textit{the weights with the largest magnitude after low-rank approximation}, representing its most influential computational pathways.
The validity of this proxy is confirmed by their perturbation studies, which show that modifying these specific weights causes sharp \textit{reasoning performance degradation}. This degradation is directly linked to high-curvature regions via a Taylor expansion of the loss. 
The \textit{principal mask}, $M_{\mathrm{princ}}^{(k)}=\mathrm{Top}_\alpha\!\big(s^{(k)}_{ij}\big)$,
is defined as the top-$\alpha$ fraction of weights with the highest score, $s^{(k)}_{ij}=\lvert W_0^{(k)}(i,j)\rvert$, where $W_0^{k}$ is the rank‑$k$ SVD reconstruction of $W_0$.

\textbf{Low-magnitude weights as low-resistance pathway.}
We further include the top-$\alpha$ lowest magnitude weights, as $M_{\mathrm{low}} = \mathrm{Bottom}_\alpha\!\big(\lvert W_0\rvert\big)$.
The magnitude is also a bias from the model geometry (distribution prior), impacting how easily the weights can be updated based on our precision gate.

\textbf{Metrics.}~~
Let $M$ be the weight update \emph{update mask} from an RLVR run.
We report the overlap ratio between our identified mask $M_\bullet$ with it, defined as $\mathrm{Overlap}(M_\bullet,M)\;=\;\frac{\lvert M_\bullet\cap M\rvert}{\lvert M\rvert}.$, with a random guess baseline overlap ratio as the density of $M_\bullet$ itself., i.e., $\alpha$.

\textbf{Our findings.}
Fig.~\ref{fig:rl-avoids-principal} visualizes the RL update mask $M$ in relation to the principal mask $M_{princ}$ and
the low-magnitude mask $M_{low}$, reporting their layer-wise overlap against a random baseline as well.
The results show a clear dichotomy. RL updates exhibit a sub-random overlap with principal weights, indicating a strong tendency to avoid them.
Conversely, the updates show a super-random overlap with low-magnitude weights due to their low resistance to micro-updates.
Besides, we found that the residual overlap between updates and principal weights is highly accounted for by
weights that are both principal (defined by the rank-k approximation of $W_0$) and low-magnitude
(original $W_0$).
After excluding this intersection, i.e., $M_{princ} \cap M_{low}^c$, the overlap drops significantly.

\textbf{Insight.}
This points to a central implication: \emph{RLVR and SFT operate in distinct optimization regions of parameter space}, even at comparable task performance. 
\emph{RLVR avoids high-curvature, principal regions, whereas SFT targets them.} 
This regional mismatch helps explain the limited transferability of SFT-oriented PEFT under RL (Sec.~\ref{sec:peft-implications}).

\subsection{RLVR Relies on Model Geometry, Disrupting Geometry Destroys the Bias}
\begin{wrapfigure}{r}{0.5\textwidth}
    \vspace{-15pt}
    \centering
    \includegraphics[width=0.98\linewidth]{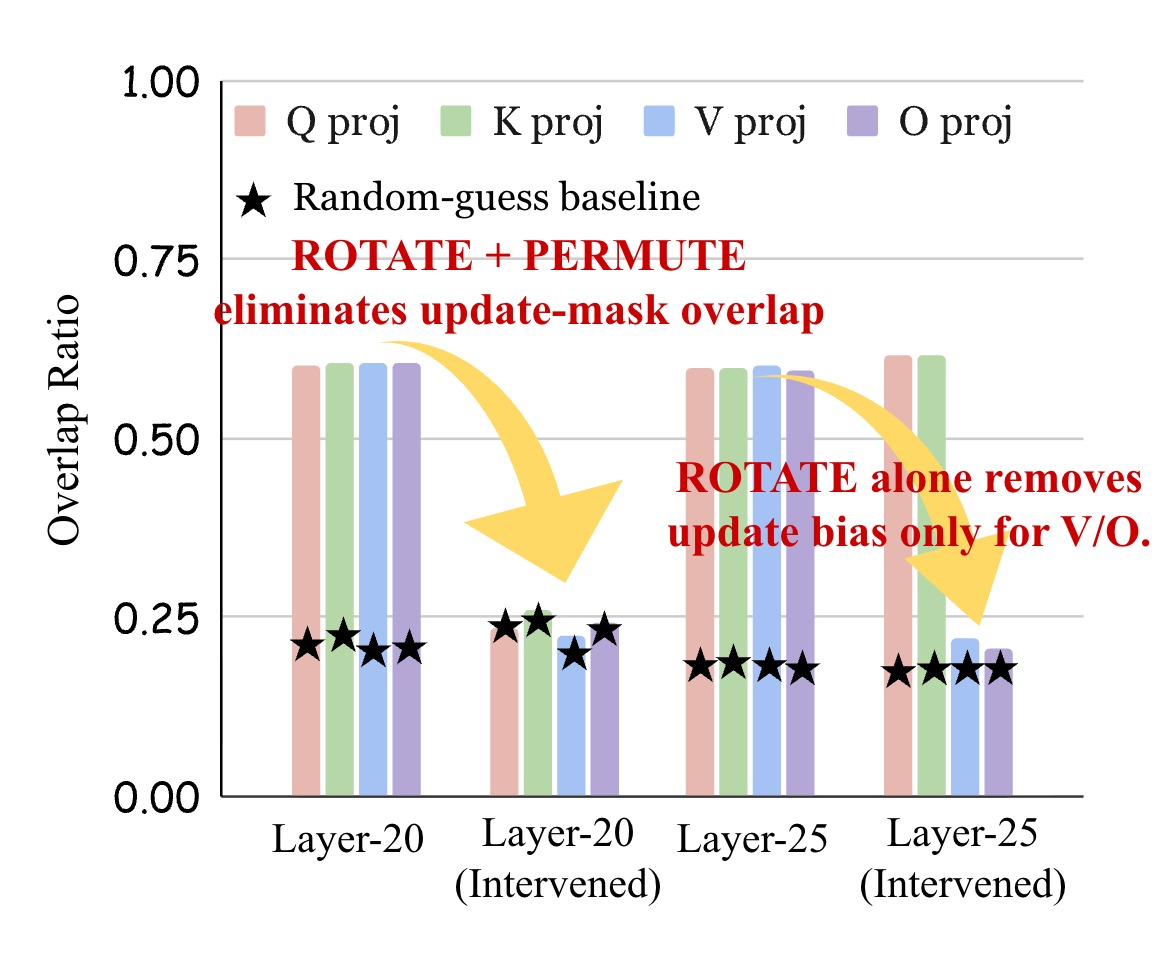}
    \vspace{-15pt}
    \caption{\small \textbf{Overlap ratio after intervention.}
    }
    \label{fig:intervention}
    \vspace{-15
    pt}
\end{wrapfigure}
Gate II posits that the pretrained model's geometry
steers RL updates. 
To test this causal link, we
deliberately
"scramble" the geometry of specific layers in a Qwen3-4B-Base model using orthogonal rotations for O/V layers (\textsc{Rotate}) and head permutations for all Q/K/V/O layers (\textsc{Permute})
(details in Appendix~\ref{appx:itervention}) and compare the update overlap ratio $\mathrm{Overlap}(M_\bullet,M)\;=\;\frac{\lvert M_\bullet\cap M\rvert}{\lvert M\rvert}.$ between the base
run with another independent run without intervention and one run with intervention.

\textbf{Our Findings.} We modify (i) layer~20 with \textsc{Rotate}$+$\textsc{Permute}, and (ii) layer~25 with \textsc{Rotate}.
As shown in Fig.~\ref{fig:intervention}, the update overlap collapsed to a random level in the intervened layers, while remaining high in all untouched layers. This provides strong causal evidence that the pretrained model's geometry is the source of the optimization bias.

\subsection{RLVR signatures persist in agentic tasks and RLHF}
\label{sec:agents-rlhf}

\begin{figure}
    \centering
    \includegraphics[width=0.80\linewidth]{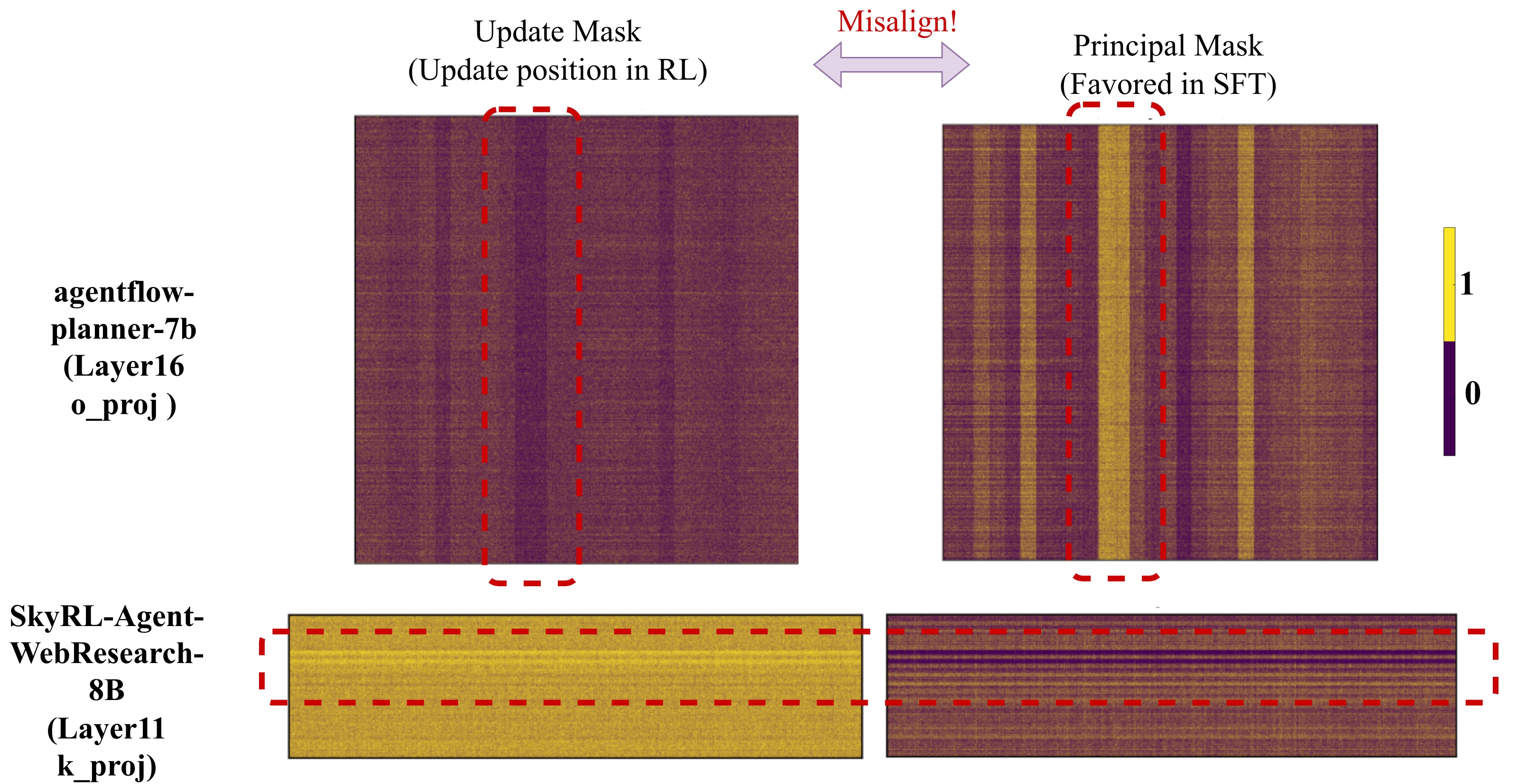}
    \vspace{-6pt}
    \caption{\small \textbf{Update–principal misalignment in RL-trained agents.}
    For representative layers from \textsc{agentflow-planner-7b} (Layer~16, $o\_{{\rm proj}}$; top row) and \textsc{SkyRL-Agent-WebResearch-8B} (Layer~11, $k\_{{\rm proj}}$; bottom row), the left panels visualize the bf16-aware \emph{update mask} $M_\ell$ (locations that changed under RL), while the right panels show the \emph{principal mask} $P^{(k)}_\ell$ (top-$k$ singular-subspace support typically favored by SFT).
    Dashed red boxes highlight stripe regions where RL updates concentrate \emph{outside} principal-weight bands,
    indicating robust off-principal routing in agent and tool-use settings.}
    \label{fig:rlvr-pos-misalign-agent}
    \vspace{-6pt}
\end{figure}

\paragraph{\textbf{Setup.}}
We analyze additional \emph{agent} and \emph{RLHF} (RL with human feedback) checkpoints and apply the same weight–space diagnostics as in
Sec.~\ref{subsec:spec-rl-sft} and Sec.~\ref{sec:princ}:
(i) principal-subspace rotation, (ii) spectral drift, and (iii) update–principal misalignment.
The extended model suite is summarized in Tab.~\ref{tab:models_more}.
\textbf{Agents.} We evaluate policies from \textsc{AgentFlow}~\citep{agentflow} and \textsc{VERL-Agent}~\citep{verlagent} on multi-turn and long-horizon tasks.
We also assess tool-augmented agents from \textsc{SkyRL}~\citep{skyrl} and \textsc{VERL-Tool}~\citep{verltool} on \textit{WebSearch}, \textit{DeepSearch}, and \textit{SWE}.
\textbf{RLHF.} We include preference-optimized models trained with DPO~\citep{dpo} and SimPO~\citep{simpo}, primarily targeting instruction following.

\paragraph{\textbf{Our Findings.}}
\textbf{(i) Stable spectra, minimal rotation.}
Across \emph{agents} and \emph{RLHF}, top-$k$ subspaces rotate only slightly, and layer spectra remain near-identical to the base model
(Fig.~\ref{fig:rl-spec-agent}; Fig.~\ref{fig:rl-spec-rlhf}), matching the spectrum-preserving, off-principal regime observed earlier.
\textbf{(ii) Off-principal updates.}
Update masks in agent and RLHF checkpoints consistently \emph{avoid principal weights}: the most active bands are
spatially misaligned with the principal mask (Fig.~\ref{fig:rlvr-pos-misalign-agent}).
\textbf{Takeaway.}
RLVR’s optimization dynamics—\emph{minimal rotation, spectrum preservation, off-principal routing}—persist beyond verifiable math/code to \emph{agents} and \emph{RLHF}, indicating a common, model-conditioned optimization bias within a KL-anchored RL post-training game, consistent with our Three-Gate Theory.

\definecolor{SFTrow}{HTML}{F3F3F3}   
\definecolor{RLrow}{HTML}{EAF2FF}    

\begin{table}
\centering
\caption{\small
\textbf{Model List for analyzed checkpoints for agentic tasks and RLHF algorithms.}
}

\label{tab:models_more}
\setlength{\tabcolsep}{4pt}
\resizebox{0.85\textwidth}{!}{%
\vspace{-10pt}
\begin{tabular}{@{} l l l l r @{}}
\toprule
\textbf{Category} & \textbf{Base Model} & \textbf{FT Model} & \textbf{Algorithm} & \textbf{Data}  \\
\midrule

\multirow{6}{*}{Agent} 
&
\href{https://huggingface.co/Qwen/Qwen3-8B}{Qwen3-8B} & \href{https://huggingface.co/NovaSky-AI/SkyRL-Agent-WebResearch-8B}{
SkyRL-Agent-WebResearch-8B} & GRPO &  WebResearch \\
&
\href{https://huggingface.co/Qwen/Qwen3-8B}{Qwen3-8B} & \href{https://huggingface.co/VerlTool/deepsearch-qwen_qwen3-8b-grpo-n16-b128-t1.0-lr1e-6-new_global_step_70}{VT-deepsearch-8B} & GRPO &  Deepsearch \\
&
\href{https://huggingface.co/Qwen/Qwen3-8B}{Qwen3-8B} & \href{https://huggingface.co/VerlTool/SWE-Qwen3-8B-VT-grpo-n32-b256-t1.0-lr2e-6}{VT-SWE-8B} & GRPO &  SWE \\
&
\href{https://huggingface.co/Qwen/Qwen2.5-7B-Instruct}{Qwen2.5-7B-Instruct} & \href{https://huggingface.co/AgentFlow/agentflow-planner-7b}{agentflow-planner-7b} & Flow-GRPO &  Planning \\
&
\href{https://huggingface.co/Qwen/Qwen2.5-7B-Instruct}{Qwen2.5-7B-Instruct} & \href{https://huggingface.co/langfeng01/GiGPO-Qwen2.5-7B-Instruct-WebShop}{GiGPO-Qwen2.5-7B-Instruct-WebShop} & GiGPO &  WebShop \\
&
\href{https://huggingface.co/Qwen/Qwen2.5-7B-Instruct}{Qwen2.5-7B-Instruct} & \href{https://huggingface.co/langfeng01/GiGPO-Qwen2.5-7B-Instruct-ALFWorld}{GiGPO-Qwen2.5-7B-Instruct-ALFWorld} & GiGPO &  ALFWorld \\
\midrule
\multirow{2}{*}{RLHF} 
&
\href{https://huggingface.co/meta-llama/Meta-Llama-3-8B-Instruct}{Meta-Llama-3-8B-Instruct} & \href{https://huggingface.co/princeton-nlp/Llama-3-Instruct-8B-DPO}{Llama-3-Instruct-8B-DPO} & DPO & instruction-following \\
&
\href{https://huggingface.co/meta-llama/Meta-Llama-3-8B-Instruct}{Meta-Llama-3-8B-Instruct} & \href{https://huggingface.co/princeton-nlp/Llama-3-Instruct-8B-SimPO}{Llama-3-Instruct-8B-SimPO} & SimPO & instruction-following \\
\bottomrule
\end{tabular}}
\vspace{-10pt}
\end{table}

\begin{figure}
    \centering
    \includegraphics[width=0.90\linewidth]{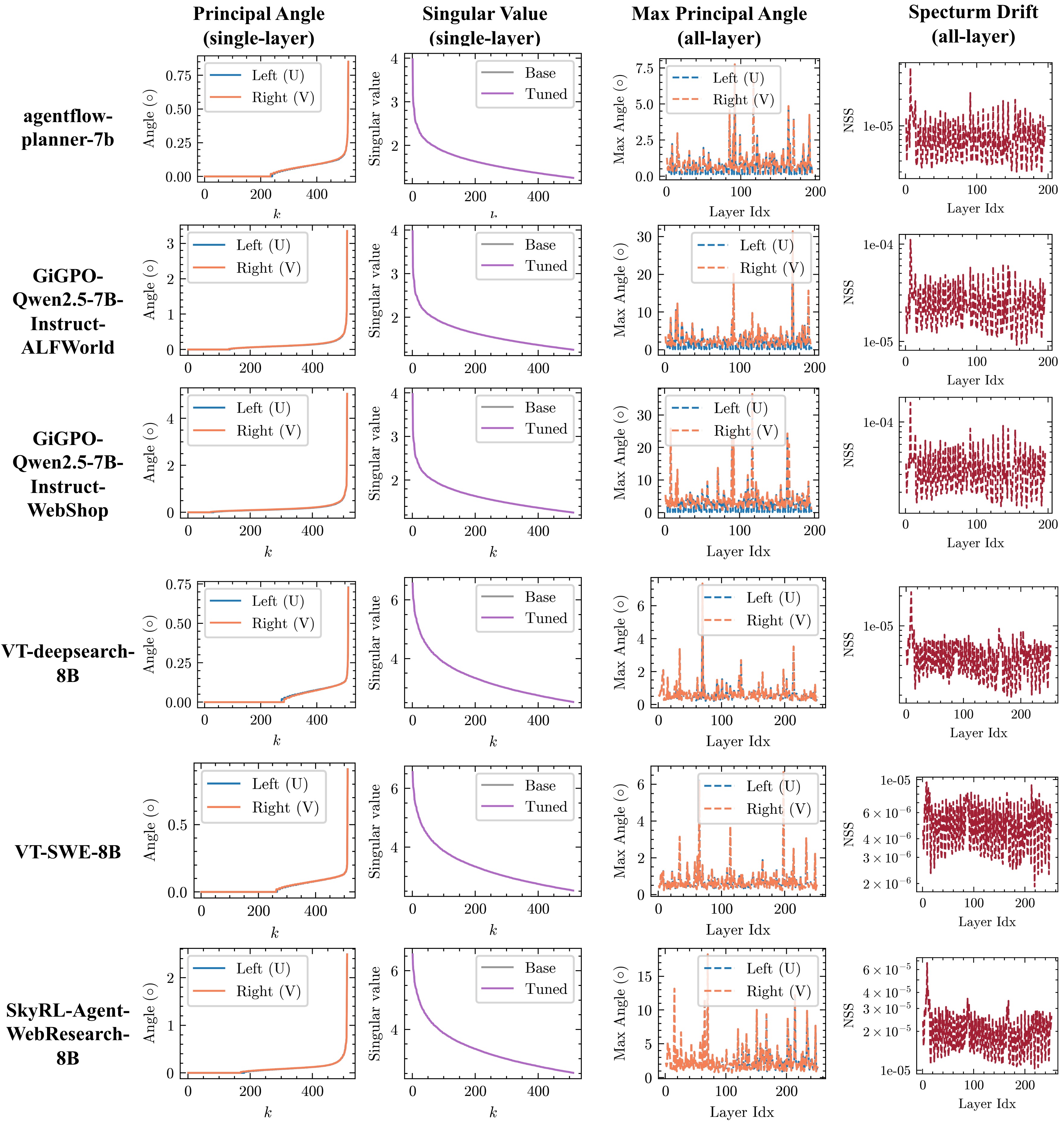}
    \vspace{-10pt}
    \caption{\small \textbf{Spectrum under RL in agent tasks.}
    In agent settings, including multi-turn interactions and tool use, RL leaves layer singular-value spectra nearly unchanged and induces only small rotations of the top-$k$ singular subspaces, consistent with the spectrum-preserving, off-principal RLVR regime. Results for RL with human feedback (RLHF), which exhibit the same optimization signature, appear in Fig.~\ref{fig:rl-spec-rlhf}.
    For consistency, we use the \emph{second block} $O$-projection layer as an exemplar single-layer readout.}
    \label{fig:rl-spec-agent}
    \vspace{-10pt}
\end{figure}

\begin{tcolorbox}[colback=orange!10!white,colframe=black,boxrule=0.9pt,boxsep=2pt,top=3pt,bottom=3pt,left=3pt,right=3pt]
\textbf{Takeaway 3.}
\textbf{RLVR learns off the principals:} it preserves spectral geometry, avoids principal weights, and its optimization bias vanishes when pretrained geometry is disrupted.
\end{tcolorbox}

\section{Theory-Guided Rethinking of Learning Algorithms for RL}
\label{sec:peft-implications}

A good theory should not only \emph{explain} observations but also \emph{inform} design.
Our account shows that RLVR and SFT follow disjoint optimization dynamics in parameter space, which implies that many SFT-era PEFT methods, especially those aligned with principal directions through sparse or low-rank priors, transfer poorly to RLVR.
This section \emph{validates} our predictions and \emph{demonstrates} how they guide the redesign of learning algorithms for RL.

\subsection{Probing Sparse Fine-Tuning in RL}

We use \textbf{sparse RL fine-tuning} to probe RL’s optimization dynamics by asking which weights can be frozen without materially altering the training trajectory.
We construct a \textbf{parameter mask} directly from the pretrained model \emph{without any additional training} and apply it to perform sparse RL fine-tuning.  
Following \citep{shenfeld2025rl}, we track the token-wise forward KL divergence $\mathrm{KL}(\pi \,\|\, \pi_{\text{ref}})$ between the fine-tuned policy and the base model throughout training.
This metric quantifies how closely a sparse run follows the dense baseline trajectory,
if pruning certain weights impedes learning, the KL drift will slow, indicating blocked optimization progress.
\begin{wrapfigure}{r}{0.5\textwidth}
    \vspace{-10pt}
    \centering
    \includegraphics[width=0.9\linewidth]{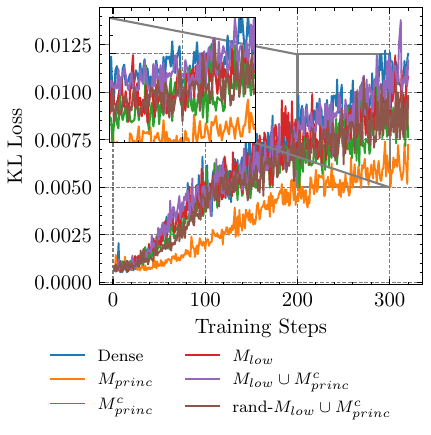}
    \vspace{-10pt}
    \caption{\small \textbf{KL loss curves on \texttt{DS-Qwen-1.5B}} under different masks.}
    \label{fig:kl-runs}
    \vspace{-10pt}
\end{wrapfigure}

\textbf{Mask design.}  
We evaluate several masks constructed directly from the pretrained model:
(i) $M_{\text{princ}}$ (\textit{principal-only}, top-50\% principal weights),  
(ii) $M_{\text{princ}}^{c}$ (\textit{non-principal-only}, the complementary subspace),  
(iii) $M_{\text{low}}$ (\textit{low-magnitude-only}),  
(iv) $M_{\text{low}}\cup M_{\text{princ}}^{c}$ (\textit{safe mask}, favoring non-principal and low-magnitude weights),
, and (v) a random mask with the same layer‑wise sparsity as (iv).
We choose 50\% for (i) as we want to isolate the effect of the number of parameters for a fair comparison to see the difference between (i) and (ii).

\textbf{Our findings.}
(See KL in Fig.~\ref{fig:kl-runs}; accuracy in Tab.~\ref{tab:benchmarks_combined} and the extended 500-step results in Tab.~\ref{tab:benchmarks_combined_500}.)
The \emph{safe} mask $M_{\text{low}}\cup M_{\text{princ}}^{c}$ most closely tracks the dense RLVR KL curve and reaches comparable final accuracy, indicating that our theory correctly identifies \emph{highly touchable} weights.
By contrast, the \emph{principal-only} mask yields the \emph{worst} optimization trajectory—its KL curve rises slowly—showing excessive intervention and degraded training dynamics.
This directly confirms that principal-targeted directions favored in SFT are ineffective for RL.

\textbf{Insight.}
Our results suggest a simple yet effective alternative:  
\textit{Freezing principal and large‑magnitude weights while updating non‑principal, low‑magnitude ones} closely reproduces dense RLVR behavior (KL trajectory and final accuracy) using roughly 70\% the parameters
This shows that our theory provides \textbf{practical guidance} for identifying the effective subspace of RL updates, \emph{entirely without additional training}. 
While the masks used here are \textit{one-shot} and fixed, combining this framework with dynamic mask refresh or adaptive scheduling~\citep{zhao2024galore, zhu2024apollo, liu2025lift} is a promising next step.

\subsection{Revisiting LoRA Through the Lens of Our Theory}

A recent report \citep{schulman2025lora} finds that low-rank LoRA, even rank-1, can match full-parameter RL performance.
Our theory offers an explanation: in full-parameter RL, effective updates lie \emph{off} the principal directions and induce only small spectral changes.
Low-rank adapters can approximate these off-principal updates, while freezing the base weights regularizes training and discourages moves toward principal directions.
With an appropriately scaled learning rate, the limited adapter capacity is therefore sufficient to catch up to full-parameter performance at least in the short run.

\textbf{However}, the same report suggests principal‑targeted variants such as \textbf{PiSSA}~\citep{meng2024pissa} should yield further gains.  
Our geometry account disagrees: aligning updates to top-$r$ principal directions enforces SFT-style behavior that is \emph{misaligned} with RLVR’s off-principal bias.

\textbf{Empirical test.}
On \texttt{DS-Qwen-1.5B} with \texttt{DeepMath-103K} \citep{he2025deepmath}, we sweep ranks $\{8,32,64\}$ and learning rates $\{1\!\times\!10^{-4},\,5\!\times\!10^{-5},\,1\!\times\!10^{-5}\}$ for 200 steps, and report $\mathrm{pass@1}$ (mean over 16 samples) on AIME24 and AMC23 (Fig.~\ref{fig:lora-pissa-1.5b}).
To control for model effects, we repeat on \texttt{Llama-3.2-3B-Instruct} with a Math corpus and report $\mathrm{pass@1}$ (mean over 4) on \textsc{MATH500} (Fig.~\ref{fig:lora-pissa-llama-3b}).

\begin{figure}
    \centering
    \includegraphics[width=0.88\linewidth]{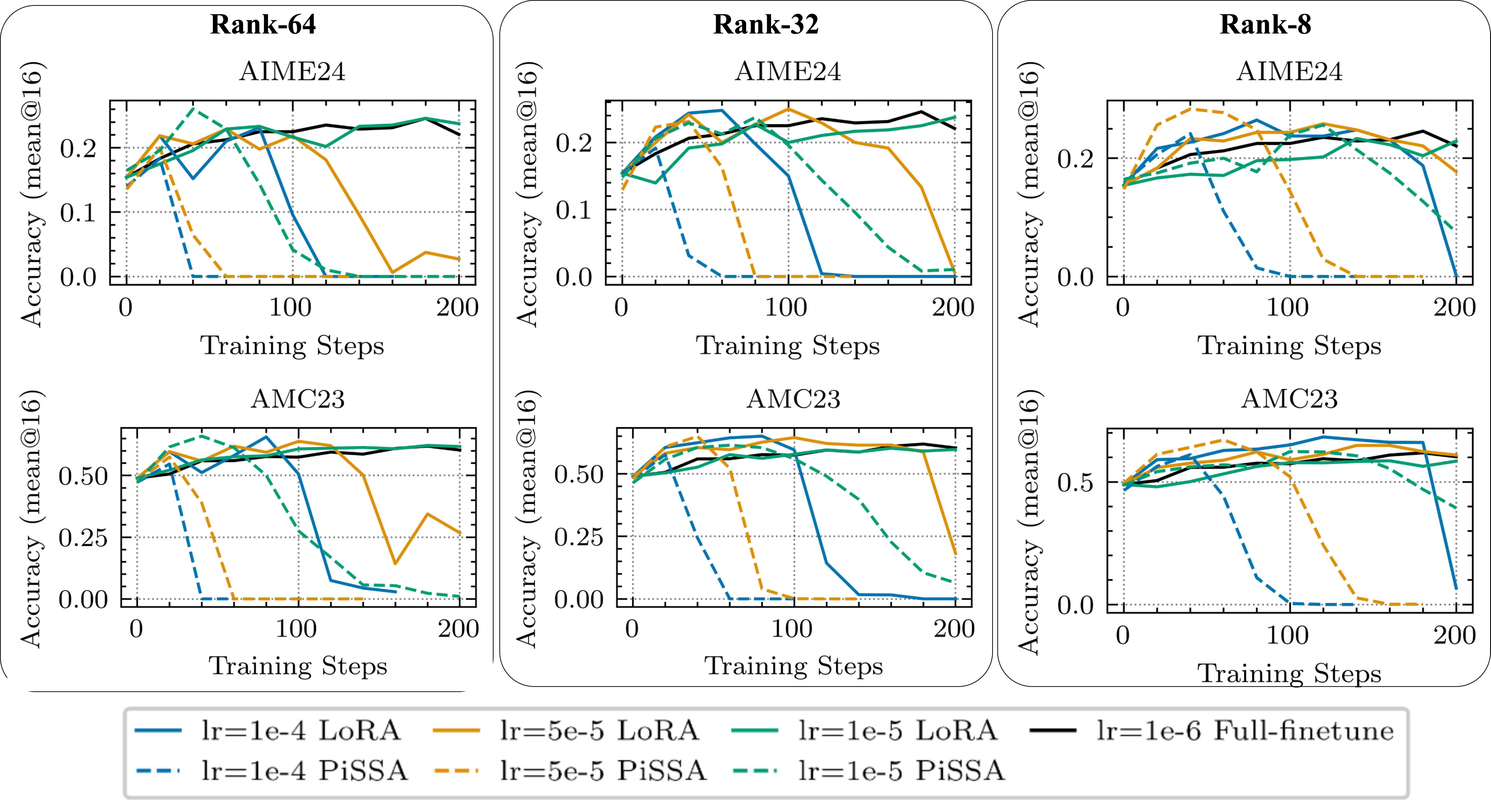}
    \caption{
    \textbf{LoRA vs. PiSSA on \texttt{DS-Qwen-1.5B} (DeepMath-103K).}
    We sweep ranks $\{8, 32, 64\}$ and learning rates $\{1\!\times\!10^{-4}, 5\!\times\!10^{-5}, 1\!\times\!10^{-5}\}$ for 200 steps, reporting $\mathrm{pass@1}$ (avg@16) on AIME24 (top) and AMC23 (bottom). 
    Across settings, \textbf{PiSSA} (principal-targeted) provides \emph{no additional gains} over LoRA and, at higher learning rates that force principal-direction updates, \emph{often collapses early}; LoRA remains more stable. 
    This supports our geometric account: forcing updates into principal directions (favored in SFT) is misaligned with RL, offering no obvious gain and leading to training collapse when scaling up learning rates.
    }
    \label{fig:lora-pissa-1.5b}
\end{figure}

\textbf{Our findings.}
Across settings, the principal-targeted \emph{PiSSA} provides no clear gain over LoRA.
At the higher learning rates used for low-rank adapters to match full-parameter performance, PiSSA \emph{often becomes unstable and collapses} earlier than LoRA.
This occurs because scaling the learning rate in PiSSA \emph{enforces updates along principal directions}, higher-curvature and spectrum-distorting, precisely the directions RLVR tends to avoid.
The result is brittle optimization and early collapse, whereas LoRA’s off-principal updates remain better aligned with RLVR’s geometry.

\textbf{Insight.}
These results support the geometry-based account: principal-aligned LoRA variants are \emph{over-fit to SFT’s update geometry} and misaligned with RL’s training dynamics, so success in SFT does not transfer to RL.

\begin{figure}
    \centering
    \includegraphics[width=0.8\linewidth]{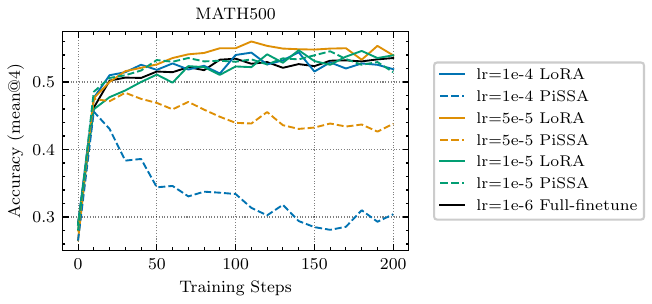}
    \caption{
    \textbf{LoRA vs.\ PiSSA on \texttt{LLaMA-3.2-3B}.}
    We sweep learning rates $\{1\!\times\!10^{-4}, 5\!\times\!10^{-5}, 1\!\times\!10^{-5}\}$ with a fixed rank of 64 for 200 steps, reporting $\mathrm{pass@1}$ (mean@4) on \textsc{MATH500}.
    Consistent with the \texttt{DS-Qwen-1.5B} results in Fig.~\ref{fig:lora-pissa-1.5b}, \textbf{PiSSA} provides \emph{no additional gain} over LoRA and, under higher learning rates that emphasize principal-direction updates, \emph{often collapses early}.
    }
    \label{fig:lora-pissa-llama-3b}
\end{figure}

\begin{tcolorbox}[colback=orange!10!white,colframe=black,boxrule=0.9pt,boxsep=2pt,top=3pt,bottom=3pt,left=3pt,right=3pt]
\textbf{Takeaway 4.} 
\textbf{RLVR operates in a distinct, geometry-driven optimization regime, so \emph{your old PEFT tricks may not work.}}  
Methods over-aligned with SFT’s principal-targeted dynamics (e.g., sparse or low-rank adapters) fail to transfer, calling for a new generation of \emph{RL-native, geometry-aware parameter-efficient algorithms.}
\end{tcolorbox}

\section{Conclusion}
We revisited the paradox of \emph{visible update sparsity} in RLVR and showed that it is a superficial readout of a deeper, \emph{model-conditioned, geometry-aligned optimization bias} that determines where updates land.
We formalized this mechanism with the \textbf{Three-Gate Theory}: a KL anchor constrains each on-policy step; pretrained geometry steers updates \emph{off principal directions} into low-curvature, spectrum-preserving subspaces; and finite precision renders the bias visible as sparsity by masking micro-updates.
Empirically, RLVR preserves spectral structure and avoids principal weights, whereas SFT targets principal directions and distorts the spectrum; when the pretrained geometry is disrupted, these signatures vanish, establishing geometry as the steering core.
Beyond explanation, our case studies bridge mechanism and practice: SFT-era principal-aligned PEFT (e.g., sparse/low-rank variants) often misaligns with RLVR’s off-principal regime.
Taken together, these results provide the \emph{first parameter-level account} of RLVR’s training dynamics, replacing a black-box view with a \emph{white-box} understanding of how parameters evolve under RLVR, and laying the foundation for \textbf{geometry-aware, RLVR-native} parameter-efficient learning algorithms.

\section*{Acknowledgment}
We thank Zhengqi Gao (Massachusetts Institute of Technology) for insightful discussion on the verl framework and idea discussion.
\newpage
\bibliography{reference/lrm, reference/new}

\begin{thebibliography}{63}
\providecommand{\natexlab}[1]{#1}
\providecommand{\url}[1]{\texttt{#1}}
\expandafter\ifx\csname urlstyle\endcsname\relax
  \providecommand{\doi}[1]{doi: #1}\else
  \providecommand{\doi}{doi: \begingroup \urlstyle{rm}\Url}\fi

\bibitem[Achiam et~al.(2023)Achiam, Adler, Agarwal, Ahmad, Akkaya, Aleman, Almeida, Altenschmidt, Altman, Anadkat, et~al.]{achiam2023gpt}
Josh Achiam, Steven Adler, Sandhini Agarwal, Lama Ahmad, Ilge Akkaya, Florencia~Leoni Aleman, Diogo Almeida, Janko Altenschmidt, Sam Altman, Shyamal Anadkat, et~al.
\newblock Gpt-4 technical report.
\newblock \emph{arXiv preprint arXiv:2303.08774}, 2023.

\bibitem[Brown et~al.(2020)Brown, Mann, Ryder, Subbiah, Kaplan, Dhariwal, Neelakantan, Shyam, Sastry, Askell, et~al.]{brown2020language}
Tom Brown, Benjamin Mann, Nick Ryder, Melanie Subbiah, Jared~D Kaplan, Prafulla Dhariwal, Arvind Neelakantan, Pranav Shyam, Girish Sastry, Amanda Askell, et~al.
\newblock Language models are few-shot learners.
\newblock \emph{Advances in neural information processing systems}, 33:\penalty0 1877--1901, 2020.

\bibitem[Cao et~al.(2025)Cao, Hegde, Li, Griggs, Liu, Tang, Pan, Wang, Malik, Hakhamaneshi, Liaw, Moritz, Zaharia, Gonzalez, and Stoica]{skyrl}
Shiyi Cao, Sumanth Hegde, Dacheng Li, Tyler Griggs, Shu Liu, Eric Tang, Jiayi Pan, Xingyao Wang, Akshay Malik, Kourosh Hakhamaneshi, Richard Liaw, Philipp Moritz, Matei Zaharia, Joseph~E. Gonzalez, and Ion Stoica.
\newblock Skyrl-v0: Train real-world long-horizon agents via reinforcement learning.
\newblock \url{https://novasky-ai.notion.site/skyrl-v0}, 2025.
\newblock NovaSky AI, Notion page. Accessed 2025-11-10.

\bibitem[Chu et~al.(2025)Chu, Zhai, Yang, Tong, Xie, Schuurmans, Le, Levine, and Ma]{chu2025sft}
Tianzhe Chu, Yuexiang Zhai, Jihan Yang, Shengbang Tong, Saining Xie, Dale Schuurmans, Quoc~V Le, Sergey Levine, and Yi~Ma.
\newblock Sft memorizes, rl generalizes: A comparative study of foundation model post-training.
\newblock \emph{arXiv preprint arXiv:2501.17161}, 2025.

\bibitem[Chung et~al.(2024)Chung, Hou, Longpre, Zoph, Tay, Fedus, Li, Wang, Dehghani, Brahma, et~al.]{chung2024scaling}
Hyung~Won Chung, Le~Hou, Shayne Longpre, Barret Zoph, Yi~Tay, William Fedus, Yunxuan Li, Xuezhi Wang, Mostafa Dehghani, Siddhartha Brahma, et~al.
\newblock Scaling instruction-finetuned language models.
\newblock \emph{Journal of Machine Learning Research}, 25\penalty0 (70):\penalty0 1--53, 2024.

\bibitem[DeepSeek-AI(2025)]{deepseekai2025deepseekr1incentivizingreasoningcapability}
DeepSeek-AI.
\newblock Deepseek-r1: Incentivizing reasoning capability in llms via reinforcement learning, 2025.
\newblock URL \url{https://arxiv.org/abs/2501.12948}.

\bibitem[Dodge et~al.(2020)Dodge, Ilharco, Schwartz, Farhadi, Hajishirzi, and Smith]{dodge2020fine}
Jesse Dodge, Gabriel Ilharco, Roy Schwartz, Ali Farhadi, Hannaneh Hajishirzi, and Noah Smith.
\newblock Fine-tuning pretrained language models: Weight initializations, data orders, and early stopping.
\newblock \emph{arXiv preprint arXiv:2002.06305}, 2020.

\bibitem[Feng et~al.(2025)Feng, Xue, Liu, and An]{verlagent}
Lang Feng, Zhenghai Xue, Tingcong Liu, and Bo~An.
\newblock Group-in-group policy optimization for llm agent training.
\newblock \emph{arXiv preprint arXiv:2505.10978}, 2025.

\bibitem[Guo et~al.(2025)Guo, Yang, Zhang, and Song]{guo2025deepseek}
Daya Guo, Dejian Yang, Haowei Zhang, and Junxiao Song.
\newblock Deepseek-r1: Incentivizing reasoning capability in llms via reinforcement learning.
\newblock \emph{arXiv preprint arXiv:2501.07570}, 2025.
\newblock URL \url{https://arxiv.org/abs/2501.07570}.

\bibitem[Han et~al.(2025)Han, Pari, Gershman, and Agrawal]{han2025general}
Seungwook Han, Jyothish Pari, Samuel~J Gershman, and Pulkit Agrawal.
\newblock General reasoning requires learning to reason from the get-go.
\newblock \emph{arXiv preprint arXiv:2502.19402}, 2025.

\bibitem[He et~al.(2025)He, Liang, Xu, Liu, Chen, Wang, Song, Yu, Liang, Wang, et~al.]{he2025deepmath}
Zhiwei He, Tian Liang, Jiahao Xu, Qiuzhi Liu, Xingyu Chen, Yue Wang, Linfeng Song, Dian Yu, Zhenwen Liang, Wenxuan Wang, et~al.
\newblock Deepmath-103k: A large-scale, challenging, decontaminated, and verifiable mathematical dataset for advancing reasoning.
\newblock \emph{arXiv preprint arXiv:2504.11456}, 2025.

\bibitem[Hendrycks et~al.(2021)Hendrycks, Burns, Kadavath, Arora, Basart, Tang, Song, and Steinhardt]{hendrycks2021measuring}
Dan Hendrycks, Collin Burns, Saurav Kadavath, Akul Arora, Steven Basart, Eric Tang, Dawn Song, and Jacob Steinhardt.
\newblock Measuring mathematical problem solving with the math dataset.
\newblock \emph{arXiv preprint arXiv:2103.03874}, 2021.

\bibitem[Howard \& Ruder(2018)Howard and Ruder]{howard2018universal}
Jeremy Howard and Sebastian Ruder.
\newblock Universal language model fine-tuning for text classification.
\newblock \emph{arXiv preprint arXiv:1801.06146}, 2018.

\bibitem[Hu et~al.(2023)Hu, Xie, Jain, Francis, Patrikar, Keetha, Kim, Xie, Zhang, Fang, et~al.]{hu2023toward}
Yafei Hu, Quanting Xie, Vidhi Jain, Jonathan Francis, Jay Patrikar, Nikhil Keetha, Seungchan Kim, Yaqi Xie, Tianyi Zhang, Hao-Shu Fang, et~al.
\newblock Toward general-purpose robots via foundation models: A survey and meta-analysis.
\newblock \emph{arXiv preprint arXiv:2312.08782}, 2023.

\bibitem[Huan et~al.(2025)Huan, Li, Zheng, Xu, Kim, Du, Poovendran, Neubig, and Yue]{huan2025does}
Maggie Huan, Yuetai Li, Tuney Zheng, Xiaoyu Xu, Seungone Kim, Minxin Du, Radha Poovendran, Graham Neubig, and Xiang Yue.
\newblock Does math reasoning improve general llm capabilities? understanding transferability of llm reasoning.
\newblock \emph{arXiv preprint arXiv:2507.00432}, 2025.

\bibitem[Jaech et~al.(2024)]{jaech2024openai}
Aaron Jaech et~al.
\newblock Openai o1 system card.
\newblock \emph{arXiv preprint arXiv:2412.16720}, 2024.

\bibitem[Jiang et~al.(2025)Jiang, Lu, Li, Lyu, Nie, Wang, Su, Chen, Zou, Du, et~al.]{verltool}
Dongfu Jiang, Yi~Lu, Zhuofeng Li, Zhiheng Lyu, Ping Nie, Haozhe Wang, Alex Su, Hui Chen, Kai Zou, Chao Du, et~al.
\newblock Verltool: Towards holistic agentic reinforcement learning with tool use.
\newblock \emph{arXiv preprint arXiv:2509.01055}, 2025.

\bibitem[Kwon et~al.(2023)Kwon, Li, Zhuang, Sheng, Zheng, Yu, Gonzalez, Zhang, and Stoica]{kwon2023efficient}
Woosuk Kwon, Zhuohan Li, Siyuan Zhuang, Ying Sheng, Lianmin Zheng, Cody~Hao Yu, Joseph~E. Gonzalez, Hao Zhang, and Ion Stoica.
\newblock Efficient memory management for large language model serving with pagedattention.
\newblock In \emph{Proceedings of the ACM SIGOPS 29th Symposium on Operating Systems Principles}, 2023.

\bibitem[Li et~al.(2024)Li, Gan, Yang, Yang, Li, Wang, Gao, et~al.]{li2024multimodal}
Chunyuan Li, Zhe Gan, Zhengyuan Yang, Jianwei Yang, Linjie Li, Lijuan Wang, Jianfeng Gao, et~al.
\newblock Multimodal foundation models: From specialists to general-purpose assistants.
\newblock \emph{Foundations and Trends{\textregistered} in Computer Graphics and Vision}, 16\penalty0 (1-2):\penalty0 1--214, 2024.

\bibitem[Li et~al.(2025)Li, Zhang, Han, Liu, Xie, Zhang, Choi, Zou, and Lu]{agentflow}
Zhuofeng Li, Haoxiang Zhang, Seungju Han, Sheng Liu, Jianwen Xie, Yu~Zhang, Yejin Choi, James Zou, and Pan Lu.
\newblock In-the-flow agentic system optimization for effective planning and tool use.
\newblock \emph{arXiv preprint arXiv:2510.05592}, 2025.

\bibitem[Lightman et~al.(2023)Lightman, Kosaraju, Burda, Edwards, Baker, Lee, Leike, Schulman, Sutskever, and Cobbe]{lightman2023lets}
Hunter Lightman, Vineet Kosaraju, Yura Burda, Harri Edwards, Bowen Baker, Teddy Lee, Jan Leike, John Schulman, Ilya Sutskever, and Karl Cobbe.
\newblock Let's verify step by step.
\newblock \emph{arXiv preprint arXiv:2305.20050}, 2023.

\bibitem[Liu et~al.(2025{\natexlab{a}})Liu, Diao, Lu, Hu, Dong, Choi, Kautz, and Dong]{liu2025prorl}
Mingjie Liu, Shizhe Diao, Ximing Lu, Jian Hu, Xin Dong, Yejin Choi, Jan Kautz, and Yi~Dong.
\newblock Prorl: Prolonged reinforcement learning expands reasoning boundaries in large language models.
\newblock \emph{arXiv preprint arXiv:2505.24864}, 2025{\natexlab{a}}.

\bibitem[Liu et~al.(2025{\natexlab{b}})Liu, Chen, Li, Qi, Pang, Du, Lee, and Lin]{liu2025understanding}
Zichen Liu, Changyu Chen, Wenjun Li, Penghui Qi, Tianyu Pang, Chao Du, Wee~Sun Lee, and Min Lin.
\newblock Understanding r1-zero-like training: A critical perspective.
\newblock \emph{arXiv preprint arXiv:2503.20783}, 2025{\natexlab{b}}.

\bibitem[Liu et~al.(2025{\natexlab{c}})Liu, Pang, Balabanov, Yang, Huang, Yin, Yang, and Liu]{liu2025lift}
Zihang Liu, Tianyu Pang, Oleg Balabanov, Chaoqun Yang, Tianjin Huang, Lu~Yin, Yaoqing Yang, and Shiwei Liu.
\newblock Lift the veil for the truth: Principal weights emerge after rank reduction for reasoning-focused supervised fine-tuning.
\newblock \emph{arXiv preprint arXiv:2506.00772}, 2025{\natexlab{c}}.

\bibitem[Loshchilov \& Hutter(2017)Loshchilov and Hutter]{loshchilov2017decoupled}
Ilya Loshchilov and Frank Hutter.
\newblock Decoupled weight decay regularization.
\newblock \emph{arXiv preprint arXiv:1711.05101}, 2017.

\bibitem[Luo et~al.(2025{\natexlab{a}})Luo, Tan, Huang, Shi, Xin, Cai, Patel, Ariyak, Wu, Zhang, Li, Popa, and Stoica]{deepcoder2025}
Michael Luo, Sijun Tan, Roy Huang, Xiaoxiang Shi, Rachel Xin, Colin Cai, Ameen Patel, Alpay Ariyak, Qingyang Wu, Ce~Zhang, Li~Erran Li, Raluca~Ada Popa, and Ion Stoica.
\newblock Deepcoder: A fully open-source 14b coder at o3-mini level.
\newblock https://pretty-radio-b75.notion.site/DeepCoder-A-Fully-Open-Source-14B-Coder-at-O3-mini-Level-1cf81902c14680b3bee5eb349a512a51, 2025{\natexlab{a}}.
\newblock Notion Blog.

\bibitem[Luo et~al.(2025{\natexlab{b}})Luo, Tan, Wong, Shi, Tang, Roongta, Cai, Luo, Li, Popa, and Stoica]{deepscaler2025}
Michael Luo, Sijun Tan, Justin Wong, Xiaoxiang Shi, William~Y. Tang, Manan Roongta, Colin Cai, Jeffrey Luo, Li~Erran Li, Raluca~Ada Popa, and Ion Stoica.
\newblock Deepscaler: Surpassing o1-preview with a 1.5b model by scaling rl, 2025{\natexlab{b}}.
\newblock Notion Blog.

\bibitem[{MAA}(2023)]{AMC23}
{MAA}.
\newblock American mathematics contest 12 (amc 12), November 2023.
\newblock URL \url{https://artofproblemsolving.com/wiki/index.php/AMC_12_Problems_and_Solutions}.

\bibitem[{MAA}(2024)]{AIME24}
{MAA}.
\newblock American invitational mathematics examination (aime), February 2024.
\newblock URL \url{https://artofproblemsolving.com/wiki/index.php/AIME_Problems_and_Solutions}.

\bibitem[{MAA}(2025)]{AIME25}
{MAA}.
\newblock American invitational mathematics examination (aime), February 2025.
\newblock URL \url{https://artofproblemsolving.com/wiki/index.php/AIME_Problems_and_Solutions}.

\bibitem[Meng et~al.(2024{\natexlab{a}})Meng, Wang, and Zhang]{meng2024pissa}
Fanxu Meng, Zhaohui Wang, and Muhan Zhang.
\newblock Pissa: Principal singular values and singular vectors adaptation of large language models.
\newblock \emph{Advances in Neural Information Processing Systems}, 37:\penalty0 121038--121072, 2024{\natexlab{a}}.

\bibitem[Meng et~al.(2024{\natexlab{b}})Meng, Xia, and Chen]{meng2024simpo}
Yu~Meng, Mengzhou Xia, and Danqi Chen.
\newblock Simpo: Simple preference optimization with a reference-free reward.
\newblock \emph{Advances in Neural Information Processing Systems}, 37:\penalty0 124198--124235, 2024{\natexlab{b}}.

\bibitem[Meng et~al.(2024{\natexlab{c}})Meng, Xia, and Chen]{simpo}
Yu~Meng, Mengzhou Xia, and Danqi Chen.
\newblock Simpo: Simple preference optimization with a reference-free reward.
\newblock \emph{Advances in Neural Information Processing Systems}, 37:\penalty0 124198--124235, 2024{\natexlab{c}}.

\bibitem[Mukherjee et~al.(2025)Mukherjee, Yuan, Hakkani-Tur, and Peng]{mukherjee2025reinforcement}
Sagnik Mukherjee, Lifan Yuan, Dilek Hakkani-Tur, and Hao Peng.
\newblock Reinforcement learning finetunes small subnetworks in large language models.
\newblock \emph{Advances in Neural Information Processing Systems}, 2025.

\bibitem[Ouyang et~al.(2022)]{ouyang2022training}
Long Ouyang et~al.
\newblock Training language models to follow instructions with human feedback.
\newblock In \emph{Advances in Neural Information Processing Systems}, volume~35, pp.\  27730--27744, 2022.

\bibitem[Radford et~al.(2018)Radford, Narasimhan, Salimans, Sutskever, et~al.]{radford2018improving}
Alec Radford, Karthik Narasimhan, Tim Salimans, Ilya Sutskever, et~al.
\newblock Improving language understanding by generative pre-training.
\newblock \emph{arXiv preprint arXiv:2303.08774}, 2018.

\bibitem[Radford et~al.(2021)Radford, Kim, Hallacy, Ramesh, Goh, Agarwal, Sastry, Askell, Mishkin, Clark, et~al.]{radford2021learning}
Alec Radford, Jong~Wook Kim, Chris Hallacy, Aditya Ramesh, Gabriel Goh, Sandhini Agarwal, Girish Sastry, Amanda Askell, Pamela Mishkin, Jack Clark, et~al.
\newblock Learning transferable visual models from natural language supervision.
\newblock In \emph{International conference on machine learning}, pp.\  8748--8763. PmLR, 2021.

\bibitem[Rafailov et~al.(2023{\natexlab{a}})Rafailov, Sharma, Mitchell, Manning, Ermon, and Finn]{dpo}
Rafael Rafailov, Archit Sharma, Eric Mitchell, Christopher~D Manning, Stefano Ermon, and Chelsea Finn.
\newblock Direct preference optimization: Your language model is secretly a reward model.
\newblock \emph{Advances in neural information processing systems}, 36:\penalty0 53728--53741, 2023{\natexlab{a}}.

\bibitem[Rafailov et~al.(2023{\natexlab{b}})Rafailov, Sharma, Mitchell, Manning, Ermon, and Finn]{rafailov2023direct}
Rafael Rafailov, Archit Sharma, Eric Mitchell, Christopher~D Manning, Stefano Ermon, and Chelsea Finn.
\newblock Direct preference optimization: Your language model is secretly a reward model.
\newblock \emph{Advances in neural information processing systems}, 36:\penalty0 53728--53741, 2023{\natexlab{b}}.

\bibitem[Raoof et~al.(2025)Raoof, Guha, Marten, Mercat, Frankel, Keh, Bansal, Smyrnis, Nezhurina, Vu, Sprague, Merrill, Chen, Choi, Khan, Grover, Feuer, Suvarna, Su, Zhao, Sharma, Ji, Arora, Li, Gokaslan, Pratt, Muennighoff, Saad-Falcon, Yang, Aali, Pimpalgaonkar, Albalak, Dave, Pouransari, Durrett, Oh, Hashimoto, Shankar, Choi, Bansal, Hegde, Heckel, Jitsev, Sathiamoorthy, Dimakis, and Schmidt]{Evalchemy}
Negin Raoof, Etash~Kumar Guha, Ryan Marten, Jean Mercat, Eric Frankel, Sedrick Keh, Hritik Bansal, Georgios Smyrnis, Marianna Nezhurina, Trung Vu, Zayne~Rea Sprague, Mike~A Merrill, Liangyu Chen, Caroline Choi, Zaid Khan, Sachin Grover, Benjamin Feuer, Ashima Suvarna, Shiye Su, Wanjia Zhao, Kartik Sharma, Charlie Cheng-Jie Ji, Kushal Arora, Jeffrey Li, Aaron Gokaslan, Sarah~M Pratt, Niklas Muennighoff, Jon Saad-Falcon, John Yang, Asad Aali, Shreyas Pimpalgaonkar, Alon Albalak, Achal Dave, Hadi Pouransari, Greg Durrett, Sewoong Oh, Tatsunori Hashimoto, Vaishaal Shankar, Yejin Choi, Mohit Bansal, Chinmay Hegde, Reinhard Heckel, Jenia Jitsev, Maheswaran Sathiamoorthy, Alex Dimakis, and Ludwig Schmidt.
\newblock Automatic evals for llms, 2025.
\newblock URL \url{https://github.com/mlfoundations/evalchemy}.

\bibitem[Ross et~al.(2011)Ross, Gordon, and Bagnell]{ross2011reduction}
St{\'e}phane Ross, Geoffrey Gordon, and Drew Bagnell.
\newblock A reduction of imitation learning and structured prediction to no-regret online learning.
\newblock In \emph{Proceedings of the fourteenth international conference on artificial intelligence and statistics}, pp.\  627--635. JMLR Workshop and Conference Proceedings, 2011.

\bibitem[Schulman \& Lab(2025)Schulman and Lab]{schulman2025lora}
John Schulman and Thinking~Machines Lab.
\newblock Lora without regret.
\newblock \emph{Thinking Machines Lab: Connectionism}, 2025.
\newblock \doi{10.64434/tml.20250929}.
\newblock https://thinkingmachines.ai/blog/lora/.

\bibitem[Shao et~al.(2024)]{shao2024deepseekmath}
Zhihong Shao et~al.
\newblock Deepseekmath: Pushing the limits of mathematical reasoning in open language models.
\newblock \emph{arXiv preprint arXiv:2402.03300}, 2024.

\bibitem[Shenfeld et~al.(2025)Shenfeld, Pari, and Agrawal]{shenfeld2025rl}
Idan Shenfeld, Jyothish Pari, and Pulkit Agrawal.
\newblock Rl's razor: Why online reinforcement learning forgets less.
\newblock \emph{arXiv preprint arXiv:2509.04259}, 2025.

\bibitem[Sheng et~al.(2024)]{sheng2024hybridflow}
Guangming Sheng et~al.
\newblock Hybridflow: A flexible and efficient rlhf framework.
\newblock \emph{arXiv preprint arXiv:2409.19256}, 2024.

\bibitem[Stewart(1998)]{stewart1998perturbation}
Gilbert~W Stewart.
\newblock Perturbation theory for the singular value decomposition.
\newblock 1998.

\bibitem[Su et~al.(2025)Su, Pan, Bai, Liu, Dong, Huang, Hu, and Zhou]{su2025klear}
Zhenpeng Su, Leiyu Pan, Xue Bai, Dening Liu, Guanting Dong, Jiaming Huang, Wenping Hu, and Guorui Zhou.
\newblock Klear-reasoner: Advancing reasoning capability via gradient-preserving clipping policy optimization.
\newblock \emph{arXiv preprint arXiv:2508.07629}, 2025.

\bibitem[Team(2025)]{qwen3technicalreport}
Qwen Team.
\newblock Qwen3 technical report, 2025.
\newblock URL \url{https://arxiv.org/abs/2505.09388}.

\bibitem[Touvron et~al.(2023)Touvron, Lavril, Izacard, Martinet, Lachaux, Lacroix, Rozi{\`e}re, Goyal, Hambro, Azhar, et~al.]{touvron2023llama}
Hugo Touvron, Thibaut Lavril, Gautier Izacard, Xavier Martinet, Marie-Anne Lachaux, Timoth{\'e}e Lacroix, Baptiste Rozi{\`e}re, Naman Goyal, Eric Hambro, Faisal Azhar, et~al.
\newblock Llama: Open and efficient foundation language models.
\newblock \emph{arXiv preprint arXiv:2302.13971}, 2023.

\bibitem[Wedin(1972)]{wedin1972perturbation}
Per-{\AA}ke Wedin.
\newblock Perturbation bounds in connection with singular value decomposition.
\newblock \emph{BIT Numerical Mathematics}, 12\penalty0 (1):\penalty0 99--111, 1972.

\bibitem[Wei et~al.(2021)Wei, Bosma, Zhao, Guu, Yu, Lester, Du, Dai, and Le]{wei2021finetuned}
Jason Wei, Maarten Bosma, Vincent~Y Zhao, Kelvin Guu, Adams~Wei Yu, Brian Lester, Nan Du, Andrew~M Dai, and Quoc~V Le.
\newblock Finetuned language models are zero-shot learners.
\newblock \emph{arXiv preprint arXiv:2109.01652}, 2021.

\bibitem[Wu et~al.(2025)Wu, Xuan, Lu, Harchaoui, and Choi]{wu2025invisible}
Fang Wu, Weihao Xuan, Ximing Lu, Zaid Harchaoui, and Yejin Choi.
\newblock The invisible leash: Why rlvr may not escape its origin.
\newblock \emph{arXiv preprint arXiv:2507.14843}, 2025.

\bibitem[{xAI}(2025)]{grok2025}
{xAI}.
\newblock Grok: Ai assistant, 2025.
\newblock URL \url{https://x.ai/grok}.
\newblock Accessed: 2025-09-24, continuously updated.

\bibitem[Xiong et~al.(2025)Xiong, Yao, Xu, Pang, Wang, Sahoo, Li, Jiang, Zhang, Xiong, et~al.]{xiong2025minimalist}
Wei Xiong, Jiarui Yao, Yuhui Xu, Bo~Pang, Lei Wang, Doyen Sahoo, Junnan Li, Nan Jiang, Tong Zhang, Caiming Xiong, et~al.
\newblock A minimalist approach to llm reasoning: from rejection sampling to reinforce.
\newblock \emph{arXiv preprint arXiv:2504.11343}, 2025.

\bibitem[Yang et~al.(2024)Yang, Zhang, Hui, Gao, Yu, Li, Liu, Tu, Zhou, Lin, Lu, Xue, Lin, Liu, Ren, and Zhang]{yang2024qwen25mathtechnicalreportmathematical}
An~Yang, Beichen Zhang, Binyuan Hui, Bofei Gao, Bowen Yu, Chengpeng Li, Dayiheng Liu, Jianhong Tu, Jingren Zhou, Junyang Lin, Keming Lu, Mingfeng Xue, Runji Lin, Tianyu Liu, Xingzhang Ren, and Zhenru Zhang.
\newblock Qwen2.5-math technical report: Toward mathematical expert model via self-improvement.
\newblock \emph{arXiv preprint arXiv:2409.12122}, 2024.

\bibitem[Yu et~al.(2025)Yu, Zhang, Zhu, Yuan, Zuo, Yue, Fan, Liu, Liu, Liu, et~al.]{yu2025dapo}
Qiying Yu, Zheng Zhang, Ruofei Zhu, Yufeng Yuan, Xiaochen Zuo, Yu~Yue, Tiantian Fan, Gaohong Liu, Lingjun Liu, Xin Liu, et~al.
\newblock Dapo: An open-source llm reinforcement learning system at scale.
\newblock \emph{arXiv preprint arXiv:2503.14476}, 2025.

\bibitem[Zeng et~al.(2025{\natexlab{a}})Zeng, Huang, Liu, Liu, He, Ma, and He]{zeng2025simplerl}
Weihao Zeng, Yuzhen Huang, Qian Liu, Wei Liu, Keqing He, Zejun Ma, and Junxian He.
\newblock Simplerl-zoo: Investigating and taming zero reinforcement learning for open base models in the wild.
\newblock \emph{arXiv preprint arXiv:2503.18892}, 2025{\natexlab{a}}.

\bibitem[Zeng et~al.(2025{\natexlab{b}})Zeng, Huang, Liu, Liu, He, Ma, and He]{zeng2025simplerlzooinvestigatingtamingzero}
Weihao Zeng, Yuzhen Huang, Qian Liu, Wei Liu, Keqing He, Zejun Ma, and Junxian He.
\newblock Simplerl-zoo: Investigating and taming zero reinforcement learning for open base models in the wild, 2025{\natexlab{b}}.
\newblock URL \url{https://arxiv.org/abs/2503.18892}.

\bibitem[Zhai et~al.(2024)Zhai, Bai, Lin, Pan, Tong, Zhou, Suhr, Xie, LeCun, Ma, et~al.]{zhai2024fine}
Simon Zhai, Hao Bai, Zipeng Lin, Jiayi Pan, Peter Tong, Yifei Zhou, Alane Suhr, Saining Xie, Yann LeCun, Yi~Ma, et~al.
\newblock Fine-tuning large vision-language models as decision-making agents via reinforcement learning.
\newblock \emph{Advances in neural information processing systems}, 37:\penalty0 110935--110971, 2024.

\bibitem[Zhang et~al.(2025)]{zhang2025srpo}
Xiaojiang Zhang et~al.
\newblock Srpo: A cross-domain implementation of large-scale reinforcement learning on llm.
\newblock \emph{arXiv preprint arXiv:2504.14286}, 2025.

\bibitem[Zhao et~al.(2024)Zhao, Zhang, Chen, Wang, Anandkumar, and Tian]{zhao2024galore}
Jiawei Zhao, Zhenyu Zhang, Beidi Chen, Zhangyang Wang, Anima Anandkumar, and Yuandong Tian.
\newblock Galore: Memory-efficient llm training by gradient low-rank projection.
\newblock \emph{arXiv preprint arXiv:2403.03507}, 2024.

\bibitem[Zhu et~al.(2024)Zhu, Zhang, Cong, Liu, Park, Chandra, Long, Pan, Wang, and Lee]{zhu2024apollo}
Hanqing Zhu, Zhenyu Zhang, Wenyan Cong, Xi~Liu, Sem Park, Vikas Chandra, Bo~Long, David~Z Pan, Zhangyang Wang, and Jinwon Lee.
\newblock Apollo: Sgd-like memory, adamw-level performance.
\newblock \emph{arXiv preprint arXiv:2412.05270}, 2024.

\bibitem[Ziegler et~al.(2019)Ziegler, Stiennon, Wu, Brown, Radford, Amodei, Christiano, and Irving]{ziegler2019fine}
Daniel~M Ziegler, Nisan Stiennon, Jeffrey Wu, Tom~B Brown, Alec Radford, Dario Amodei, Paul Christiano, and Geoffrey Irving.
\newblock Fine-tuning language models from human preferences.
\newblock \emph{arXiv preprint arXiv:1909.08593}, 2019.

\end{thebibliography}
\bibliographystyle{iclr2026_conference}

\appendix

\section{Clarification of LLM Usage}
In this work, we employ LLMs to polish the writing throughout the paper and to assist in generating code for figure plotting.
Besides, we use it for drawing the teaser figure.

\section{More Related works}

\paragraph{Post-training}
Large-scale models pre-trained on broad domains serve as general-purpose backbones with extensive domain knowledge and notable zero-shot capabilities~\citep{radford2021learning,achiam2023gpt,touvron2023llama,hu2023toward,li2024multimodal,radford2018improving,brown2020language}. 
However, such pre-trained models often fail to meet the specific application requirements or align with domain-specific constraints. \emph{Post-training} methods address this gap by adapting foundation models to downstream tasks. Common approaches include supervised fine-tuning on curated datasets~\citep{howard2018universal,dodge2020fine,wei2021finetuned,chung2024scaling}, reinforcement learning from human or automated feedback~\citep{ziegler2019fine,ouyang2022training,guo2025deepseek,zhai2024fine}, and other recent techniques~\citep{rafailov2023direct}. 

Especially, the recent advances in LLM reasoning~\citep{deepseekai2025deepseekr1incentivizingreasoningcapability} highlight the effectiveness of \emph{Reinforcement Learning with Verifiable Rewards} (RLVR), which replaces subjective human judgments with automatically verifiable signals. RLVR has been shown to significantly enhance reasoning ability using policy optimization algorithms such as PPO~\citep{ouyang2022training} and GRPO~\citep{shao2024deepseekmath}. Building on these successes, a growing body of work~\citep{yu2025dapo,liu2025understanding,deepcoder2025,zhang2025srpo,liu2025prorl,xiong2025minimalist} continues to refine RL methods tailored for LLM reasoning.

\paragraph{SFT versus RL.}
Prior work comparing these paradigms has largely focused on downstream performance. A foundational result shows that on-policy RL can outperform offline SFT even with the same expert data \citep{ross2011reduction}. 
Recent empirical studies consistently reinforce this, finding that RL-tuned models often generalize better out-of-distribution \citep{han2025general, chu2025sft} and transfer more effectively to new tasks \citep{huan2025does} than their SFT counterparts.

While these studies establish a performance hierarchy, our work investigates a different dimension: how these distinct methods affect the model's internal structure. 
A recent study observed that RL fine-tunes only a fraction of the network's parameters \citep{mukherjee2025reinforcement}, but this empirical finding left the underlying mechanism unexplored and did not characterize or predict the affected subnetwork. Our work aims to bridge this gap by providing a mechanistic explanation for this phenomenon.

\section{Experimental Details}

\subsection{Training Settings}
\label{appx:training}
\paragraph{Models \& Datasets.}
We run post‑training experiments on three open models:
\textbf{DeepSeek‑R1‑Distill‑Qwen‑1.5B}~\citep{yang2024qwen25mathtechnicalreportmathematical},
\textbf{Qwen2.5‑Math‑7B}~\citep{yang2024qwen25mathtechnicalreportmathematical}, and
\textbf{Qwen3‑Base}~\citep{qwen3technicalreport}.
The maximum context length is set to $8192$ for DeepSeek‑R1‑Distill‑Qwen‑1.5B and Qwen2.5‑Math‑7B, and to $20480$for Qwen3‑4B‑Base.

We evaluate primarily on mathematics using two training corpora to reduce dataset‑specific confounds.
(1) \textbf{DAPO+MATH (DM):} a union of the DAPO‑Math‑17k set\footnote{\href{https://huggingface.co/datasets/BytedTsinghua-SIA/DAPO-Math-17k/blob/main/README.md}{DAPO‑Math‑17k}}
and the MATH dataset~\citep{hendrycks2021measuring}.
(2) \textbf{DS+SR:} the 47k DeepScaler collection~\citep{deepscaler2025} combined with high‑difficulty (levels 3–5) problems extracted from SimpleRL~\citep{zeng2025simplerl}.We use the version from \citet{huan2025does}.

\paragraph{Training details.}
We implement RLVR on the VeRL pipeline~\citep{sheng2024hybridflow} (v4.0) and use vLLM~\citep{kwon2023efficient}(v8.5) for rollouts.
We use FSDPv2 with the default mixed precision configuration.
All experiments run on NVIDIA H200 GPUs.
Unless otherwise noted, we use DAPO~\citep{yu2025dapo} \emph{without} an explicit reference‑KL penalty (ratio clipping as in DAPO), a global batch size of 256 (mini‑batch 64) with 4 gradient update per step.
We use \textbf{DAPO} primarily to eliminate the confounding effect of the KL penalty, which can otherwise obscure the intrinsic parameter update dynamics during training.

Per‑model configurations without specific mention:
\begin{itemize}[leftmargin=1.25em,itemsep=2pt]
    \item \textbf{Qwen2.5‑Math‑7B} on \textbf{DM}: 16 rollouts per prompt; 8 x H200 GPUs; 300 training steps.
    \item \textbf{DeepSeek‑R1‑Distill‑Qwen‑1.5B} on \textbf{DS+SR}: 12 rollouts per prompt; 16 x H200 GPUs; 320 steps.
    \item \textbf{Qwen3‑4B‑Base} on \textbf{DS+SR}: 16 rollouts per prompt; 32 x H200 GPUs; 150 steps. 
\end{itemize}

For LoRA and PiSSA studies, to reduce compute cost during learning-rate sweeps, we use the same DAPO recipe with a global batch size of 128 (mini-batch 32), four gradient updates per step, and 16 rollouts per prompt.  
Both \textbf{DeepSeek-R1-Distill-Qwen-1.5B} and \textbf{LLaMA-3.2-3B} are trained for 200 steps.

The actor is optimized with AdamW~\citep{loshchilov2017decoupled} (constant learning rate $1{\times}10^{-6}$, $\beta_1{=}0.9$, $\beta_2{=}0.999$).  
Rewards are \emph{verifiable}: $+1.0$ if the extracted final answer is correct and $-1.0$ otherwise (no separate format score), following the verifier implementation of~\cite{su2025klear}.  
We enable an over-length penalty with an additional 1024-token budget and a penalty factor of $1.0$.

\subsection{Evaluation settings}
\label{appx:eva}

We evaluate models on four widely used benchmarks: AIME24~\citep{AIME24}, AIME25~\citep{AIME25}, AMC23~\citep{AMC23}, MATH-500~\citep{lightman2023lets}, as we main train using math daastets. We used Eval-Chemy \citep{Evalchemy} with their default temperature 0.7 and 0.8 as the top-p value. In our experiments, we used \textbf{the averaged accuracy}, i.e., $pass@1 (avg@k)$ for all benchmarks. to evaluate the models' performance. Specifically, for AIME24 and AIME 25, we averaged accuracy on 64 samples, for AMC, we average accuracy on 32 samples, For MATH 500, our score is the average accuracy over 2 samples.

\section{Intervention details}
\label{appx:itervention}

\paragraph{Intervention~1: loss–preserving V/O rotation.}
Let $D$ be the head dimension, $H_q$ the number of query heads, $H_{kv}$ the number of key/value heads, and $n_{\mathrm{rep}}=H_q/H_{kv}$ (grouped GQA). Denote
\[
W_v \in \mathbb{R}^{d_{\mathrm{model}} \times (H_{kv}D)}, \qquad
W_o \in \mathbb{R}^{d_{\mathrm{model}}\times (H_qD)} .
\]

Draw any orthogonal $R\in\mathbb{R}^{D\times D}$ (Haar/Hadamard) and form the block rotations
\[
R_{kv} \;=\; \mathrm{diag}(\underbrace{R,\ldots,R}_{H_{kv}})\in\mathbb{R}^{(H_{kv}D)\times(H_{kv}D)}, \qquad
R_q \;=\; \mathrm{diag}(\underbrace{R,\ldots,R}_{n_{\mathrm{rep}}},\,\underbrace{R,\ldots,R}_{n_{\mathrm{rep}}},\ldots)\in\mathbb{R}^{(H_qD)\times(H_qD)}.
\]
We edit the weights by right–multiplication along the head axis:
\begin{equation}
\label{eq:vo-rot}
\boxed{\, W_v' = W_v R_{kv}, \qquad W_o' = W_o R_q . \,}
\end{equation}
If $b_v$ exists, reshape $b_v$ per head and set $b_v' = b_v R_{kv}$.

\begin{proposition}[Exact invariance]
\label{prop:vo-invariance}
Let $\mathrm{Ctx}=\mathrm{Attn}(Q,K,V)\in\mathbb{R}^{\cdot\times (H_qD)}$. Under~\eqref{eq:vo-rot},
\[
\mathrm{out}' \;=\; \mathrm{Attn}\!\left(Q,K,V R_{kv}\right)\,(W_o R_q)^\top 
\;=\; \mathrm{Ctx}\,R_q R_q^\top W_o^\top \;=\; \mathrm{Ctx}\,W_o^\top \;=\; \mathrm{out}.
\]
\end{proposition}

\paragraph{Intervention~2: head shuffle (lossless).}
Let $P_{kv}$ be a permutation of the $H_{kv}$ KV heads and $P_q$ its grouped expansion to $H_q$ heads. Apply
\[
\text{cols of }(W_k,W_v)\leftarrow P_{kv}, \qquad
\text{cols of }W_q\leftarrow P_q, \qquad
\text{columns of }W_o\leftarrow P_q^{-1}.
\]
This relabels which head carries which subspace, while leaving the block function unchanged.

We show that after weight intervention, the model weights update position has a sub-random overlap while those untouched weights stay a high overlap.

\section{Examples of why previous identified method fails}

\subsection{Failures of a Fixed Absolute Tolerance Rule}
\label{app:bf16diff}

\begin{itemize}[leftmargin=1.2em]
\item \textbf{False positives at large scale.} 
Within $[2^{10},2^{11}){=}[1024,2048)$, the bf16 spacing is $\operatorname{ULP}_{\mathrm{bf16}}=2^{10-7}=8$. 
Numbers like $1024.001$ and $1024.002$ differ by $10^{-3}{>}10^{-5}$, hence would be flagged as ``changed'' by the $10^{-5}$ rule, yet both round to the same bf16 code ($1024$), i.e., \emph{no storage-level change}.
\item \textbf{False negatives at small scale.}
Around $10^{-6}\!\approx\!2^{-20}$, the bf16 spacing is $\operatorname{ULP}_{\mathrm{bf16}}=2^{-27}\!\approx\!7.45{\times}10^{-9}$. 
Weights $w{=}10^{-6}$ and $\widehat{w}{=}2{\times}10^{-6}$ differ by $10^{-6}{\le}10^{-5}$ and would be marked ``equal'' by the $10^{-5}$ rule, yet they are separated by $\approx 134$ ULPs and quantize to \emph{different} bf16 codes.
\end{itemize}

\subsection{Justification of our probe}
\label{app:bf16diff-therom}

\begin{lemma}[Gap between distinct bf16 representables]
\label{lem:gap}
If $x\neq y$ are normalized bf16 numbers in the same binade $[2^{e},2^{e+1})$, then
\[
|x-y|\;\ge\;2^{\,e-7}
\quad\text{and}\quad
\frac{|x-y|}{\max(|x|,|y|)}\;>\;2^{-8}.
\]
The strict inequality also holds across the binade boundary.
\end{lemma}

\begin{lemma}[ULP lens: magnitude-dependent threshold]
\label{lem:ULP-lens}
For normalized bf16 values $x$ with $|x|\in[2^e,2^{e+1})$,
\[
\frac{\mathrm{ULP}_{\mathrm{bf16}}(x)}{|x|}\in(2^{-8},\,2^{-7}] \;=\; (0.390625\%,\,0.78125\%].
\]
Hence the \emph{minimal realized} relative update at magnitude $|x|$ is \(\gtrsim\tfrac12\,\mathrm{ULP}_{\mathrm{bf16}}(x)/|x|\in(0.195\%,\,0.391\%]\).
In particular, \emph{larger} $|x|$ requires a \emph{larger} absolute step to register. \qed
\end{lemma}

\begin{proposition}[Soundness and completeness of the probe]
\label{prop:sound}
Let $w_i,\widehat w_i$ be \emph{normalized} bf16 values (finite, nonzero), and suppose
\(\eta < \tfrac{1}{2}\min_x \operatorname{ULP}_{\mathrm{bf16}}(x)/|x| = 2^{-9}\approx 1.953\cdot 10^{-3}\).
Then
\[
\bigl|\,\widehat{w}_i - w_i\,\bigr| \le \eta \max(|w_i|,|\widehat{w}_i|)
\quad\Longleftrightarrow\quad
\mathrm{bf16}(w_i) = \mathrm{bf16}(\widehat{w}_i).
\]
\end{proposition}

\begin{proof}
($\Rightarrow$)If $w_i\neq \widehat w_i$, Lemma~\ref{lem:ULP-lens} gives
\(|\widehat w_i-w_i|/\max(|w_i|,|\widehat w_i|) > 2^{-8} > 2\eta\), contradiction. 
Hence $w_i=\widehat w_i$ as bf16 numbers.

($\Leftarrow$) If the stored bf16 values are equal, the difference is $0$, which satisfies \eqref{eq:bf16-close}. \qed
\end{proof}

\begin{corollary}[Choice \(\eta=10^{-3}\) is safe]
\label{cor:eta}
Since \(10^{-3} < 2^{-9}\), Proposition~\ref{prop:sound} applies: the test
\eqref{eq:bf16-close} passes \emph{iff} the two bf16 entries are bit-wise identical (or both zero).
Thus $\eta=10^{-3}$ yields a \emph{scale-aware} probe that flags equality \underline{only} when storage is unchanged.
\end{corollary}

\section{Math Analysis}
\label{appx:theory}

\subsection{Policy-Gradient Fine-Tuning (DAPO)}
\label{sec:rl_ce}

Assume an \textit{old} policy
$\pi_{\text{old}}$ that we use to sample $G$ candidate completions
$y^{1:G}$ for each prompt $x\in \mathcal{X}$.
For a single token $y_{i,t}$ (token $t$ in completion $i$) we define the
\emph{importance-weighted advantage}

\[
w_{i,t}
  \;=\;
  \underbrace{\frac{\pi_\theta(y_{i,t}|x, y_{<t})}{\pi_{\text{old}}(y_{i,t}|x, y_{<t})}}_{\text{importance ratio}}
  \;\hat A_{i,t}\;
  \mathbb I_{\text{clip}}
  \quad\in\mathbb R,
\tag{1}
\]
where $\hat A_{i,t}$ is the estimated advantage and
$\mathbb I_{\text{clip}}\!\in\{0,1\}$ implements the usual trust-region clipping.

\paragraph{Token-level objective.}
The DAPO loss can be written as a sum of weighted log-probabilities
\[
J_{\text{RL}}(\theta)
  \;=\;
  \mathbb E_{x\sim\mathcal X,\;y^{1:G}\sim\pi_{\text{old}}}
  \Bigl[
    \tfrac{1}{\sum_i |y^i|}
    \sum_{i=1}^G \sum_{t=1}^{|y^i|}
      w_{i,t}\,
      \log\pi_\theta(y_{i,t}\mid x,y^i_{<t})
  \Bigr].
\tag{2}
\]

\subsection{Proof of Gate I: On‑Policy RL Implies a One‑Step KL Leash}
\label{appx:proof_gate_1}
This appendix provides the standard tilting oracle and $M$‑projection facts, local second‑order expansions, and the proof of the one‑step policy‑KL leash (Prop.~\ref{prop:one-step-kl-main} in the main text). 
\textit{\textbf{We keep the proof concise, otherwise too lengthy, especially for those has shown in some prior work~\cite{shenfeld2025rl, wu2025invisible}.}}
Our one-step analysis is inspired by recent work~\cite{wu2025invisible, shenfeld2025rl}, which uses a similar variational approach to show that even the final converged policy remains KL-proximal to the base policy.
We also record a trust‑region/clipping bound used when $\beta=0$.

Throughout, $x$ is fixed, $q(\cdot\!\mid x)$ has full support on $\mathcal{Y}$, and $\pi_\theta(\cdot\!\mid x)$ is a $C^3$ parametric family with log‑density $\log \pi_\theta$ locally smooth. Expectations without explicit subscript are conditional on $x$.

We first show useful lemmas here.

\begin{lemma}[Frozen-policy surrogate is second-order tight]
\label{lem:frozen-second-order}
Let $f(\theta):=\mathcal L_{\mathrm{PG}}(\theta)$ in \eqref{eq:g1-obj} and
$g(\theta):=\widetilde{\mathcal L}_{\mathrm{PG}}(\theta;\theta_t)$ be the frozen‑policy surrogate with $A_{\theta_t}$.
Then $f(\theta_t)=g(\theta_t)$ and $\nabla f(\theta_t)=\nabla g(\theta_t)$.
If $\nabla f$ and $\nabla g$ are $L$‑Lipschitz in a neighborhood of $\theta_t$, then
\[
|\,f(\theta_t+\Delta\theta)-g(\theta_t+\Delta\theta)\,| \;\le\; \tfrac{L}{2}\,\|\Delta\theta\|^2.
\]
\emph{Proof.}
At $\theta_t$, both objectives evaluate to $-\E_{\pi_{\theta_t}}[A_{\theta_t}\log\pi_{\theta_t}]$. For the gradient, using the log‑derivative trick and the centering of $A_{\theta_t}$, both yield
$-\E_{\pi_{\theta_t}}[A_{\theta_t}\nabla\log\pi_{\theta_t}]$.
Thus $f(\theta_t)=g(\theta_t)$ and $\nabla f(\theta_t)=\nabla g(\theta_t)$.
The bound is the standard second‑order Taylor remainder under Lipschitz gradients. \qed
\end{lemma}

\paragraph{1: Exponential tilting and M‑projection}
\label{appx:proof_kl}

\begin{lemma}[Gibbs variational principle / exponential tilting]
\label{lem:tilt}
Fix $\beta>0$ and a full‑support reference $q(\cdot\!\mid x)$. Then
\[
\max_{\pi\ll q}\ \Big\{\E_{y\sim \pi}\![R(x,y)]-\beta\,D_{\mathrm{KL}}(\pi\|q)\Big\}
\]
is uniquely maximized by
\[
\tilde q_\beta(y\mid x)=\frac{q(y\mid x)\exp(R(x,y)/\beta)}{\E_{y\sim q}[\exp(R(x,y)/\beta)]}.
\]
\end{lemma}

\begin{proof}
Consider $\mathcal L(\pi,\lambda)=\E_\pi[R]-\beta \E_\pi\!\big[\log\frac{\pi}{q}\big]+\lambda(\sum_y\pi(y)-1)$. Stationarity in $\pi$ gives 
$\log\frac{\pi}{q}=R/\beta-\lambda-1$, hence $\pi\propto q\,e^{R/\beta}$. Strict concavity in $\pi$ yields uniqueness. 
\end{proof}

\begin{lemma}[Policy Gradient Update as Parametric $M$‑projection]
\label{lem:mproj}
For fixed $\tilde q_\beta$, 
\[
\arg\min_{\theta}\ D_{\mathrm{KL}}(\tilde q_\beta\|\pi_\theta)\;=\;
\arg\max_{\theta}\ \E_{y\sim \tilde q_\beta}[\log \pi_\theta(y\mid x)].
\]
\end{lemma}

\begin{proof}
$D_{\mathrm{KL}}(\tilde q_\beta\|\pi_\theta)=\E_{\tilde q_\beta}[\log \tilde q_\beta]-\E_{\tilde q_\beta}[\log \pi_\theta]$, where the first term is $\theta$‑independent. 
We omit the full proof here, with one can be found in~\cite{shenfeld2025rl}.
\end{proof}

\paragraph{2: Local second‑order identities}
\label{appx:mproj}

\begin{lemma}[Local Pythagorean identity for the $M$‑projection]
\label{lem:local-pyth}
Let $f(\theta):=D_{\mathrm{KL}}(\tilde q_\beta\|\pi_\theta)=\E_{\tilde q_\beta}[-\log\pi_\theta]+\mathrm{const}$. 
Assume $\log\pi_\theta$ is $C^3$ near $\theta$, and let $\theta^+\in\arg\min f$. 
Writing $\Delta:=\theta^+-\theta$, for $\|\Delta\|$ small,
\[
f(\theta)-f(\theta^+)=\tfrac12\,\Delta^\top H_{\tilde q}(\theta)\,\Delta+O(\|\Delta\|^3),
\quad
H_{\tilde q}(\theta):=-\,\E_{\tilde q_\beta}[\nabla^2\log\pi_\theta].
\]
\end{lemma}

\begin{proof}
Taylor‑expand $f$ at $\theta^+$: $f(\theta)=f(\theta^+)+\tfrac12 \Delta^\top H_{\tilde q}(\theta^+)\Delta+O(\|\Delta\|^3)$ since $\nabla f(\theta^+)=0$. Local $C^3$ smoothness implies $H_{\tilde q}(\theta^+)=H_{\tilde q}(\theta)+O(\|\Delta\|)$, which is absorbed into the cubic remainder.
\end{proof}

\begin{lemma}[Quadratic expansion of policy KL]
\label{lem:quadkl}
Let $F(\theta):=-\,\E_{\pi_\theta}[\nabla^2 \log \pi_\theta]$ be the Fisher information. Then
\[
D_{\mathrm{KL}}(\pi_{\theta+\Delta}\|\pi_\theta)
=\tfrac12\,\Delta^\top F(\theta)\,\Delta+O(\|\Delta\|^3).
\]
\end{lemma}

\begin{proof}
Expand $\log \frac{\pi_{\theta+\Delta}}{\pi_\theta}=\Delta^\top\nabla\log\pi_\theta+\tfrac12\Delta^\top\nabla^2\log\pi_\theta\,\Delta+O(\|\Delta\|^3)$,
take expectation under $\pi_{\theta+\Delta}=\pi_\theta+O(\|\Delta\|)$, use $\E_{\pi_\theta}[\nabla\log\pi_\theta]=0$ and $-\E_{\pi_\theta}[\nabla^2\log\pi_\theta]=F(\theta)$.
\end{proof}

\paragraph{3. Relating projection Hessian and Fisher under small tilt}
\label{appx:hess-fisher}

\begin{lemma}[Hessian–Fisher proximity]
\label{lem:hess-fisher}
Suppose $\|\nabla^2\log\pi_\theta(y\mid x)\|_{\mathrm{op}}\le L$ uniformly near $\theta$. Then
\[
\big\|H_{\tilde q}(\theta)-F(\theta)\big\|_{\mathrm{op}}
\ \le\ 2L\,\mathrm{TV}(\tilde q_\beta,\pi_\theta)
\ \le\ L\sqrt{2\,D_{\mathrm{KL}}(\tilde q_\beta\|\pi_\theta)}.
\]
In particular, with $\kappa:=D_{\mathrm{KL}}(\tilde q_\beta\|\pi_\theta)\to 0$, we have $H_{\tilde q}(\theta)=(1+O(\sqrt{\kappa}))\,F(\theta)$ as quadratic forms.
\end{lemma}

\begin{proof}
For bounded matrix‑valued $h$, $\|\E_{\tilde q}h-\E_{\pi}h\|_{\mathrm{op}}\le 2\|h\|_\infty\,\mathrm{TV}(\tilde q,\pi)$. Apply this with $h:=-\nabla^2\log\pi_\theta$ and Pinsker’s inequality $\mathrm{TV}(p,q)\le \sqrt{\tfrac12 D_{\mathrm{KL}}(p\|q)}$.
\end{proof}

\paragraph{4. Remainder control}
\label{appx:remainder}

\begin{lemma}[Cubic remainder is $o(f)$]
\label{lem:remainder}
If $H_{\tilde q}(\theta)\succeq m I$ on the update subspace (local strong convexity), then for $\|\Delta\|$ small
\[
\|\Delta\|^2 \ \le\ \tfrac{2}{m}\,\big(f(\theta)-f(\theta^+)\big),
\qquad
O(\|\Delta\|^3)=o\!\big(f(\theta)\big).
\]
\end{lemma}

\begin{proof}
From Lemma~\ref{lem:local-pyth}, $f(\theta)-f(\theta^+)\ge \tfrac{m}{2}\|\Delta\|^2+O(\|\Delta\|^3)$. Rearranging yields $\|\Delta\|^2=O(f(\theta)-f(\theta^+))$, so the cubic term is lower order.
\end{proof}

\subsubsection{Proof of Proposition~\ref{prop:one-step-kl-main}}
\label{appx:one-step-kl-main}

\begin{proof}[Proof of Proposition~\ref{prop:one-step-kl-main}]
Let $f(\theta)=D_{\mathrm{KL}}(\tilde q_\beta\|\pi_\theta)$ and $\Delta=\theta^+-\theta$. By Lemma~\ref{lem:local-pyth},
\[
f(\theta)-f(\theta^+)=\tfrac12\,\Delta^\top H_{\tilde q}(\theta)\Delta+O(\|\Delta\|^3).
\]
By Lemma~\ref{lem:quadkl},
\[
D_{\mathrm{KL}}(\pi_{\theta^+}\|\pi_\theta)=\tfrac12\,\Delta^\top F(\theta)\Delta+O(\|\Delta\|^3).
\]
By Lemma~\ref{lem:hess-fisher} with $\kappa=f(\theta)$,
$\Delta^\top F \Delta = (1+O(\sqrt{\kappa}))\,\Delta^\top H_{\tilde q}\Delta$. Hence
\[
D_{\mathrm{KL}}(\pi_{\theta^+}\|\pi_\theta)
=(1+O(\sqrt{\kappa}))\,\big(f(\theta)-f(\theta^+)\big)\;+\;O(\|\Delta\|^3).
\]
Since $f(\theta^+)\ge 0$, $f(\theta)-f(\theta^+)\le f(\theta)=\kappa$. By Lemma~\ref{lem:remainder}, $O(\|\Delta\|^3)=o(f(\theta))$. Therefore
\[
D_{\mathrm{KL}}(\pi_{\theta^+}\|\pi_\theta)\ \le\ (1+o(1))\,f(\theta)\ =\ (1+o(1))\,D_{\mathrm{KL}}(\tilde q_\beta\|\pi_\theta),
\]
which is the desired inequality.
\end{proof}

\subsubsection{Proof of Proposition~\ref{prop:policy-to-weight-main}}
\label{appx:policy-to-weight-main}

\begin{proof}[Proof of Proposition~\ref{prop:policy-to-weight-main}]
By the quadratic expansion of policy KL (Lemma~\ref{lem:quadkl}),
\begin{equation}
\label{eq:quad-kl}
D_{\mathrm{KL}}(\pi_{\theta+\Delta}\|\pi_\theta)
= \tfrac12\,\Delta^\top F(\theta)\Delta \;+\; R(\Delta),
\qquad
|R(\Delta)|\ \le\ C\,\|\Delta\|^3
\end{equation}
for some local constant $C>0$ (from $C^3$ smoothness). Let $a:=\Delta^\top F(\theta)\Delta$.
Using the spectral lower bound $F(\theta)\succeq \mu I$ on the update subspace,
\begin{equation}
\label{eq:mu-lb}
\|\Delta\|^2 \ \le\ \tfrac{a}{\mu}.
\end{equation}
Combining \eqref{eq:quad-kl}–\eqref{eq:mu-lb} yields
\[
D_{\mathrm{KL}}(\pi_{\theta+\Delta}\|\pi_\theta)
\ \ge\ \tfrac12\,a \;-\; C\,\Big(\tfrac{a}{\mu}\Big)^{3/2}.
\]
Since $D_{\mathrm{KL}}(\pi_{\theta^+}\|\pi_\theta)\le K$, we have
\begin{equation}
\label{eq:ineq-a}
K \ \ge\ \tfrac12\,a \;-\; C\,\mu^{-3/2} a^{3/2}.
\end{equation}
For $a$ sufficiently small (equivalently, $K$ small), the cubic term is dominated by the linear term:
choose $a_0>0$ so that $C\,\mu^{-3/2}\sqrt{a}\le \tfrac14$ whenever $0<a\le a_0$.
Then from \eqref{eq:ineq-a}
\[
K \ \ge\ \Big(\tfrac12-\tfrac14\Big)a \ =\ \tfrac14\,a
\quad\Rightarrow\quad
a \ \le\ 4K.
\]
Substituting $a\le 4K$ back into \eqref{eq:quad-kl} refines the remainder:
$|R(\Delta)| \le C\|\Delta\|^3 \le C (a/\mu)^{3/2} = O(K^{3/2}) = o(K)$,
so $D_{\mathrm{KL}}(\pi_{\theta+\Delta}\|\pi_\theta) = \tfrac12 a + o(K)$.
Hence $a = 2\,D_{\mathrm{KL}}(\pi_{\theta+\Delta}\|\pi_\theta) + o(K) \le 2K + o(K)$, i.e.
\[
\Delta^\top F(\theta)\Delta \ \le\ 2K\,(1+o(1)).
\]
Taking square roots gives the Fisher‑norm bound in \eqref{eq:kl-leash-weight}:
$\|\Delta\|_{F(\theta)} = \sqrt{\Delta^\top F(\theta)\Delta} \le \sqrt{2K}\,(1+o(1))$.
The Euclidean bound follows from \eqref{eq:mu-lb}:
\[
\|\Delta\|_2 \ \le\ \sqrt{\tfrac{\Delta^\top F(\theta)\Delta}{\mu}}
\ \le\ \sqrt{\tfrac{2K}{\mu}}\,(1+o(1)).
\]
Finally, for any parameter block $W\subset\theta$, its Frobenius change is the $\ell_2$‑norm of the corresponding subvector of $\Delta$; therefore $\|\Delta W\|_F \le \|\Delta\|_2$.
\end{proof}

\subsubsection{One‑step KL budget (used in Gate~II)}
\label{appx:quadkl}

\begin{corollary}[KL budget]
\label{cor:budget}
If $D_{\mathrm{KL}}(\pi_{\theta^+}\|\pi_\theta)\le K$, then
\[
\tfrac12\,\Delta^\top F(\theta)\,\Delta \ \le\ K\,(1+o(1)).
\]
\end{corollary}

\begin{proof}
Apply Lemma~\ref{lem:quadkl} and Lemma~\ref{lem:remainder}.
\end{proof}

\subsubsection{Trust‑region / clipping bound (for $\beta=0$)}
\label{appx:proof_clip}

\begin{lemma}[Implicit KL leash from ratio clipping]
\label{lem:clip}
Let $r_t=\frac{\pi_{\theta^+}(y_t\mid x,y_{<t})}{\pi_{\theta}(y_t\mid x,y_{<t})}$ and suppose clipping enforces $r_t\in[1-\varepsilon,\,1+\varepsilon]$ on the batch. Then
\[
\widehat{D}_{\mathrm{KL}}(\pi_{\theta^+}\|\pi_\theta)
\ \le\ \widehat{\E}[T(x)]\cdot \max\{-\log(1-\varepsilon),\,\log(1+\varepsilon)\}
\ =\ O(\varepsilon)\cdot \widehat{\E}[T(x)],
\]
and in the small‑step regime (mean‑zero advantage) this tightens to $O(\varepsilon^2)$.
\end{lemma}

\begin{proof}
Autoregressive factorization gives $D_{\mathrm{KL}}(\pi_{\theta^+}\|\pi_\theta)=\E_{\pi_{\theta^+}}[\sum_t \log r_t]$. Because $\log r_t\in[\log(1-\varepsilon),\log(1+\varepsilon)]$, we have 
$|\log r_t|\le c(\varepsilon)$; summing over $t$ and taking batch expectation yields the stated bound. Using $\log(1\pm\varepsilon)=\pm\varepsilon+O(\varepsilon^2)$ and small‑step arguments gives $O(\varepsilon^2)$.
\end{proof}

\subsection{Proofs for Gate II (Sec.~\ref{sec:gate2-model})}
\label{appx:gate2-lemmas}

\paragraph{Setup (layer-conditioned budget).}
Partition $\theta=(\mathrm{vec}(W),\theta_{\neg W})$ and let the Fisher at $\theta=\theta_t$ be
\[
F(\theta)=\begin{bmatrix}F_{W,W}&F_{W,\neg W}\\F_{\neg W,W}&F_{\neg W,\neg W}\end{bmatrix}\succeq0.
\]
For a one‑step update $\Delta\theta$, the global KL leash implies
$\tfrac12\,\Delta\theta^\top F(\theta)\Delta\theta \le K$. Define the layer‑conditioned curvature
\[
S_W:=F_{W,W}-F_{W,\neg W}F_{\neg W,\neg W}^{-1}F_{\neg W,W}\succeq0,
\]
and the per‑layer budget
$\delta_W:=\tfrac12\,\mathrm{vec}(\Delta W)^\top S_W\,\mathrm{vec}(\Delta W)\le K$.
Let $\mu_W:=\lambda_{\min}(S_W)>0$ on the update subspace.

\begin{lemma}[Layer‑conditioned Frobenius/operator bounds]
\label{lem:kl-frob-op-schur}
$\|\Delta W\|_F \le \sqrt{2\delta_W/\mu_W}$ and $\|\Delta W\|_2 \le \|\Delta W\|_F$.
\end{lemma}
\begin{proof}
Since $S_W\succeq \mu_W I$, $\delta_W \ge \tfrac12\,\mu_W\|\Delta W\|_F^2$. \qed
\end{proof}

\begin{lemma}[Wedin’s sin--$\Theta$]
\label{lem:wedin}
For $W_+=W_0+\Delta W$, let $\gamma_k := \sigma_k(W_0) - \sigma_{k+1}(W_0)$ be the singular value gap. For any $k$ where $\gamma_k > 0$, the principal subspace angles satisfy
$\|\sin\Theta(U_k(W_0),U_k(W_+))\|_2 \le \|\Delta W\|_2/\gamma_k$ and similarly for $V_k$. \qed
\end{lemma}

\begin{lemma}[Weyl/Mirsky and Hoffman--Wielandt]
\label{lem:weyl-hw}
$|\sigma_k(W_+)-\sigma_k(W_0)|\le \|\Delta W\|_2$ and
$\sum_i(\sigma_i(W_+)-\sigma_i(W_0))^2 \le \|\Delta W\|_F^2$. \qed
\end{lemma}

\begin{corollary}[Projection stability]
\label{cor:proj-stability}
With the same assumptions (including $\gamma_k > 0$),
\[
\big\|U_k(W_0)U_k(W_0)^\top - U_k(W_+)U_k(W_+)^\top\big\|_2
\;=\; \big\|\sin\Theta\!\big(U_k(W_0),U_k(W_+)\big)\big\|_2
\;\le\; \frac{\sqrt{2\delta_W/\mu_W}}{\gamma_k}.
\]
The analogous bound holds for the right subspaces with $V_k$.
\emph{Interpretation.} The leading invariant subspaces rotate by at most
$O\!\big(\sqrt{\delta_W/\mu_W}/\gamma_k\big)$; when the gap is moderate, the rotation is small.
\qed
\end{corollary}

\section{More Visualization}

\subsection{Jaccard matrix}
RL updates are highly consistent across independent training runs. 
Fig.~\ref{fig:jaccard} shows the pair-wise Jaccard similarity between the final update masks from five RLVR runs on different data and algorithms. 
The high similarity scores demonstrate that the optimization process consistently targets the same subset of parameters, providing strong evidence for a deterministic, non-random optimization bias.

\begin{figure}[t]
    \centering
    \includegraphics[width=0.8\linewidth]{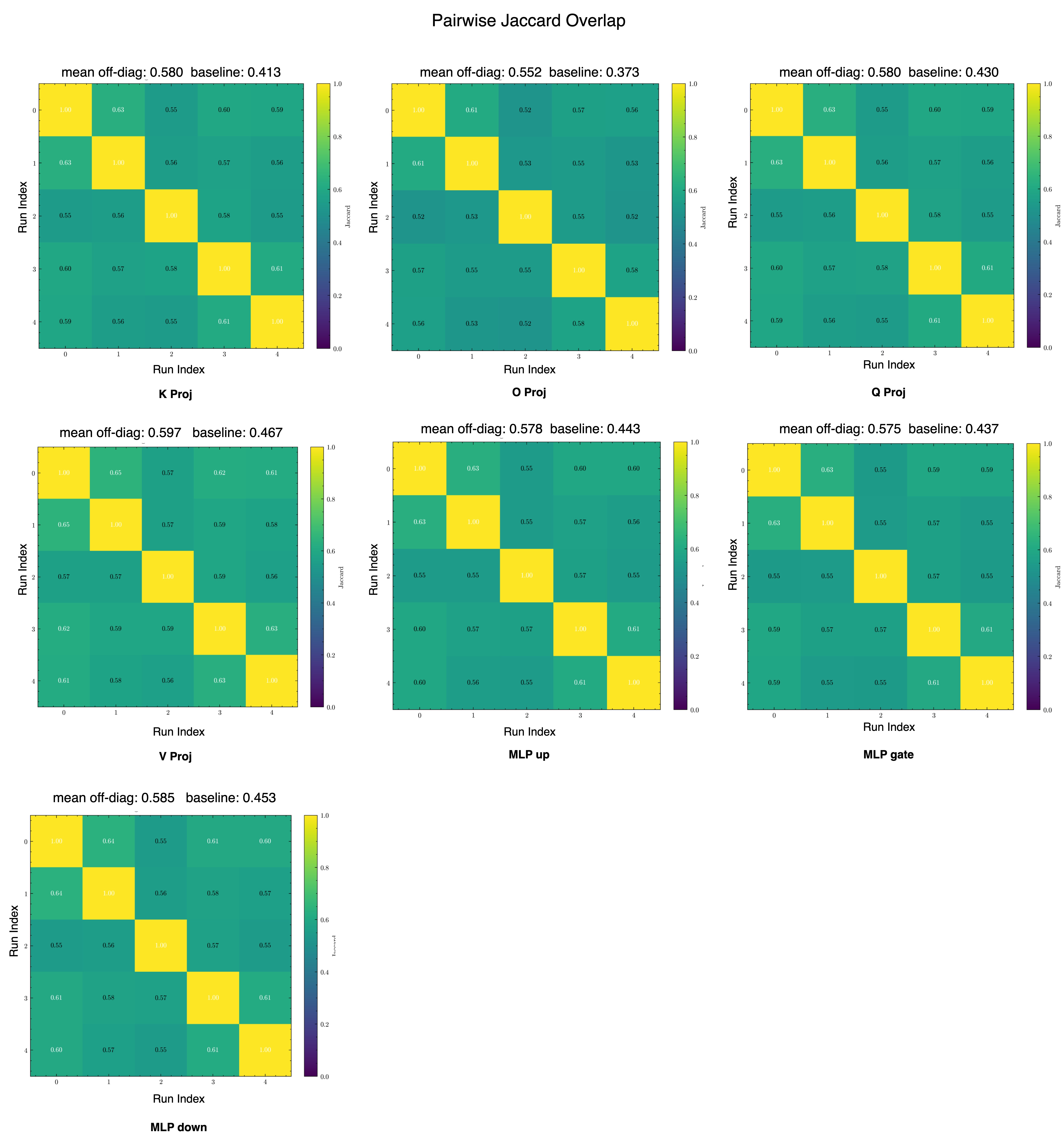}
    \vspace{-10pt}
    \caption{Pair-wise Jaccard similarity of update masks from five independent RLVR runs on Layer 13 of the \texttt{DS-Distill-Qwen-1.5B} model.
    }
    \label{fig:jaccard}
\end{figure}

\begin{figure}[h]
    \centering
    \includegraphics[width=0.75\linewidth]{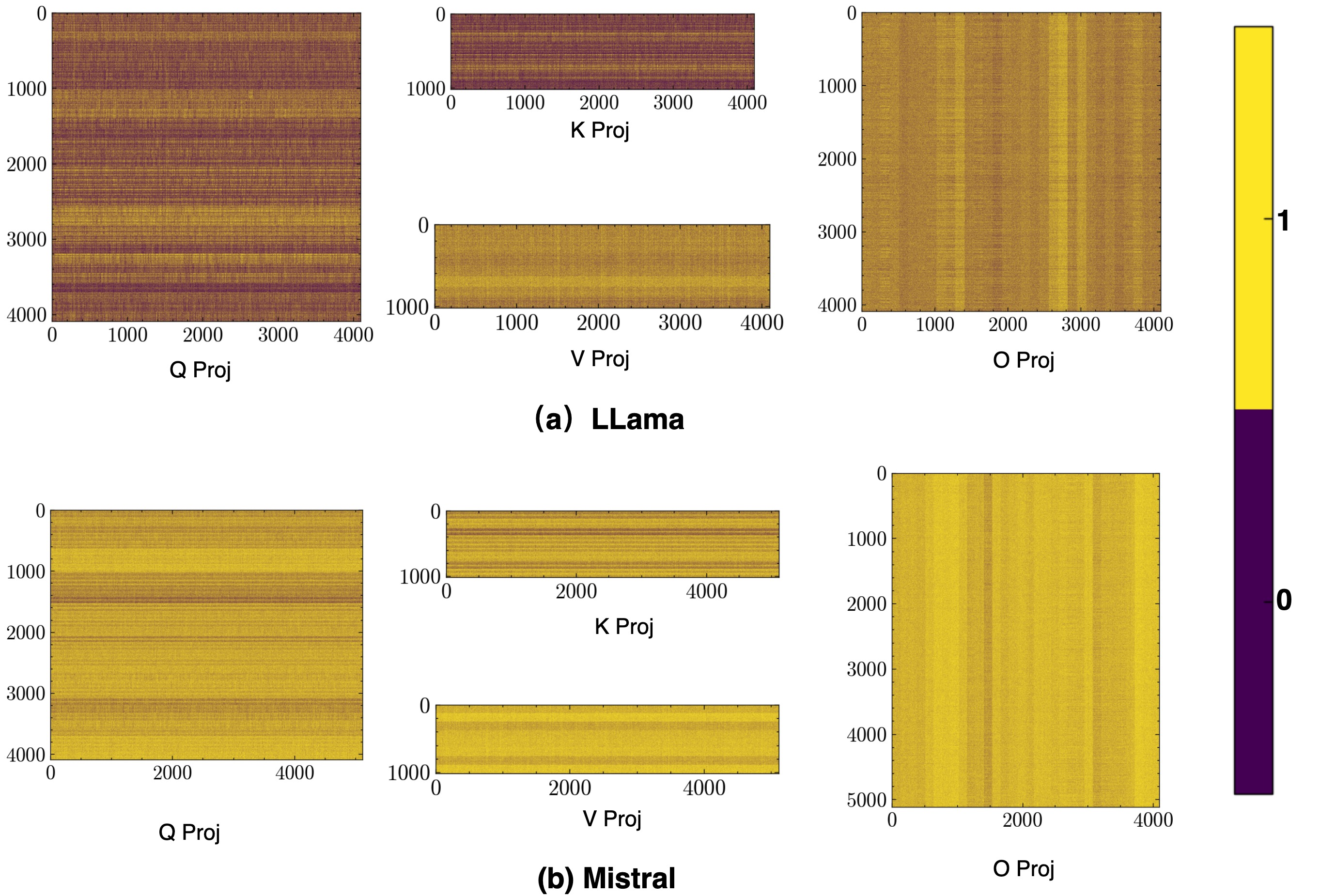}
    \caption{ 
    Structured Update observed on Llama(Llama-3.1-8B) and Mistral (Mistral-Small-24B) models. Here we plot the weight update mask using the zero-RL checkpoints from~\cite{zeng2025simplerlzooinvestigatingtamingzero}.
    }
    \label{fig:extra_llama}
\end{figure}

\begin{figure}[h]
    \vspace{-10pt}
    \centering
    \includegraphics[width=\linewidth]{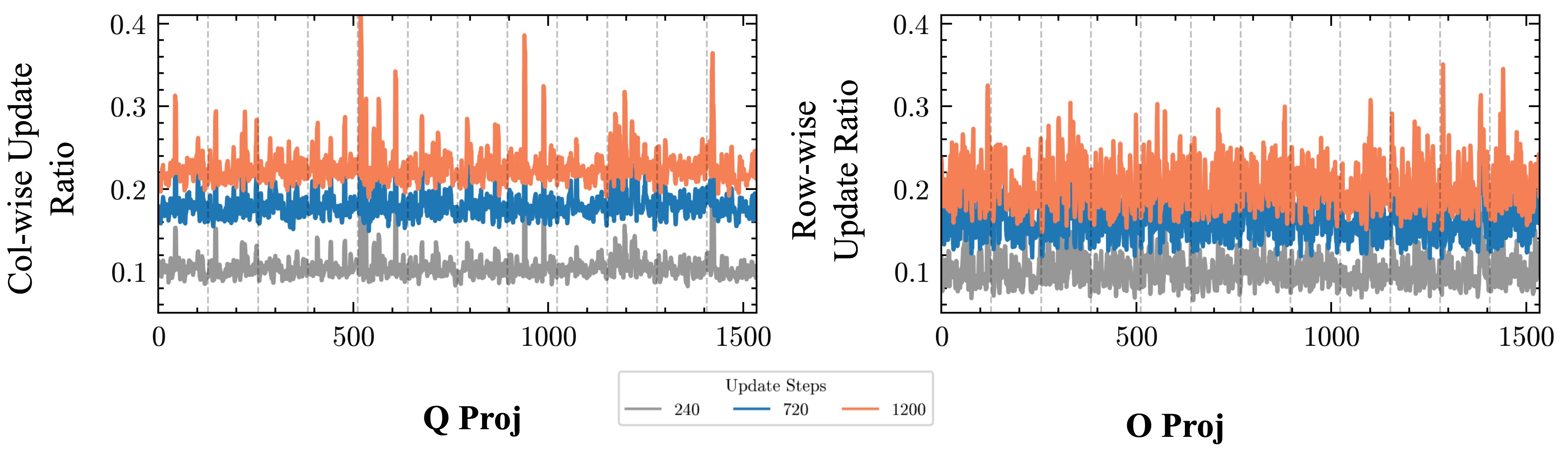}
    \vspace{-15pt}
    \caption{Temporal emergence of the optimization bias with row and column-wise update ratios for the 13th attention block across gradient update steps ($t\!\in\!\{240,720,1200\}$), smoothed with a 3-step window. \textit{The column-wise (Q) and row-wise (O) update ratios show a much weaker bias.}
    }
    \label{fig:rowwise_dynamics_comp}
    \vspace{-10pt}
\end{figure}

\subsection{Spectrum shift for DS-1.5B and Qwen3-1}
We also show the spectrum shift for DS-1.5B and Qwen3-14B here.

\label{appx:spec_more}

\begin{table}[!t]
  \centering
  \small
  \caption{
  Performance of DS-Qwen-1.5B with different masking strategies at 320 steps. Parameter counts shown are for linear layers only, excluding the embedding and head layers. Detailed evaluation settings are available in Appendix~\ref{appx:eva}.  \textbf{\textit{We observe that training only on principal weights $M_{princ}$ results in a clear accuracy gap compared to both the dense baseline and its complement $M_{princ}^c$.
  The models using the $M_{low}$ and $M_{princ}^c \cup M_{lowest}$ masks achieve performance closest to the dense baseline.}}}
  \label{tab:benchmarks_combined}
  \resizebox{\textwidth}{!}{%
    \begin{tabular}{l|c|cccc|c|c}
      \toprule
      \textbf{Model}       & \textbf{Mask}                & \textbf{Math500} & \textbf{AMC23} & \textbf{AIME24} & \textbf{AIME25} & \textbf{Average} & \textbf{\#params} \\
      \midrule
      \multirow{6}{*}{DS-Qwen-1.5B} & Dense & 84.20 & \textbf{81.56} & \textbf{36.98} & \textbf{27.03} & \textbf{57.44} & 100\% \\
       & $M_{princ}$ & 83.60 & 77.19 & 30.16 & 24.32 & 53.82 & 50\%  \\
       & $M_{princ}^c$ & 82.70 & 78.90 & 34.28 & 25.73 & 55.40 & 50\% \\
       & $M_{low}$ & 84.50 & 80.08 & 35.62 & 26.56 & 56.69 & 58.59\% \\
       & $M_{princ}^c \cup M_{low}$ & \textbf{85.20} & 78.83 & 34.74 & 26.20 & 56.24 & 74.02\% \\
       & Random-$M_{princ}^c \cup M_{low}$ & 84.50 & 77.35 & 34.48 & 25.01 & 55.34 & 74.02\% \\
      \bottomrule
    \end{tabular}%
  }
\end{table}

\begin{table}[!t]
  \centering
  \small
  \caption{
  Performance of DS-Qwen-1.5B with different masking strategies with a extended training window to 500 steps. Parameter counts shown are for linear layers only, excluding the embedding and head layers. Detailed evaluation settings are available in Appendix~\ref{appx:eva}.  \textbf{\textit{We observe that training only on principal weights $M_{princ}$ results in a clear accuracy gap compared to both the dense baseline and its complement $M_{princ}^c$.
  The models using the $M_{low}$ and $M_{princ}^c \cup M_{lowest}$ masks achieve performance closest to the dense baseline.}}}
  \label{tab:benchmarks_combined_500}
  \resizebox{\textwidth}{!}{%
    \begin{tabular}{l|c|cccc|c|c}
      \toprule
      \textbf{Model}       & \textbf{Mask}                & \textbf{Math500} & \textbf{AMC23} & \textbf{AIME24} & \textbf{AIME25} & \textbf{Average} & \textbf{\#params} \\
      \midrule
      \multirow{6}{*}{DS-Qwen-1.5B} & Dense & \textbf{84.5} & \textbf{83.52} & 38.28 & \textbf{28.075} & \textbf{58.59} & 100\% \\
       & $M_{princ}$ & 83.60 & 78.83 & 34.06 & 25.63 & 55.44 & 50\%  \\
       & $M_{princ}^c$ & 84.0 & 77.97 & 38.64 & 27.81 & 56.90 & 50\% \\
       & $M_{low}$ & 83.8 & 82.42 & 37.03 & 27.82 & 57.77 & 58.59\% \\
       & $M_{princ}^c \cup M_{low}$ & 84.10 & 81.41 & \textbf{40.30}  & 27.70 & 58.37 & 74.02\% \\
       & Random-$M_{princ}^c \cup M_{low}$ & 84.10 & 81.72 & 34.69 & 27.34 & 56.89 & 74.02\% \\
      \bottomrule
    \end{tabular}%
  }
\end{table}

\begin{figure}[h]
    \centering
    \includegraphics[width=0.5\linewidth]{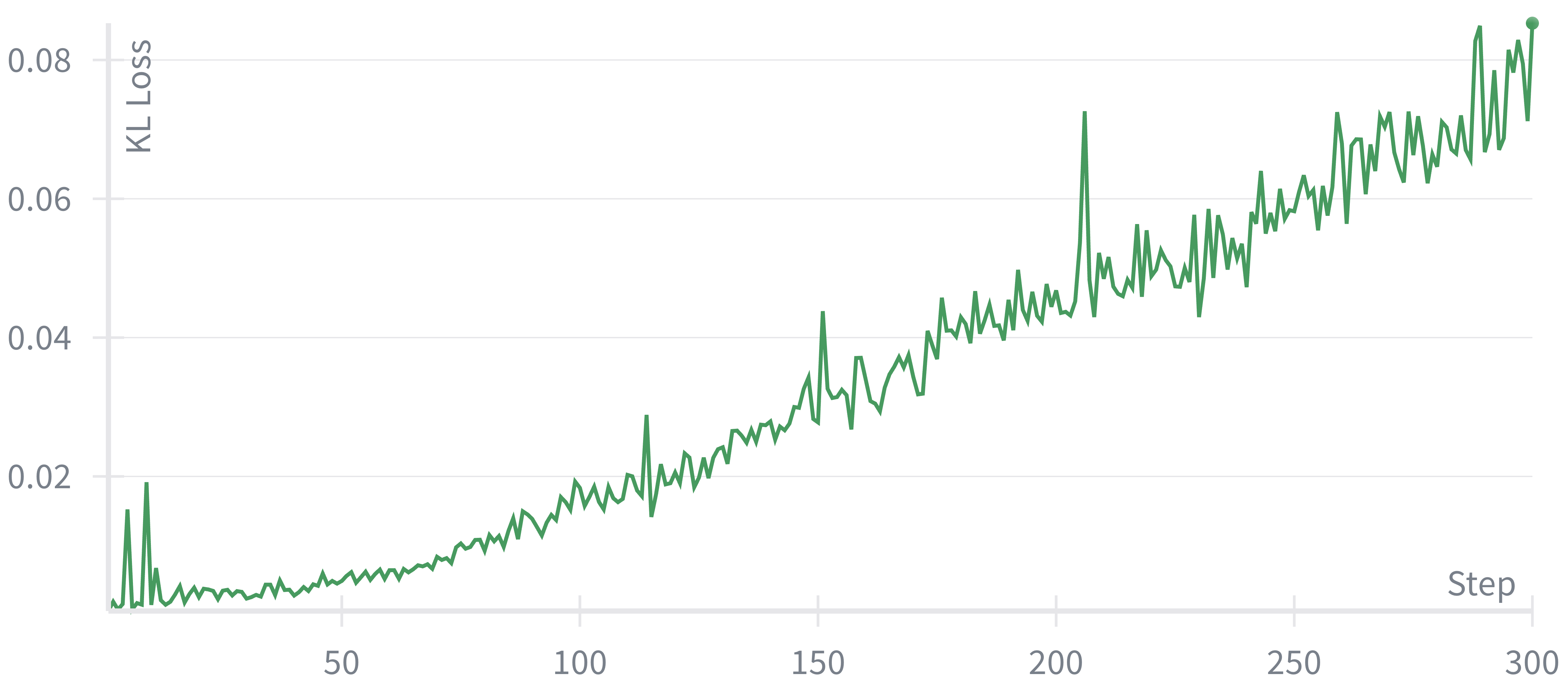}
    \caption{Token-wise KL loss. We show the token-wise KL loss during a DAPO run without a KL loss penalty, which shows a steadily increasing KL loss instead of being unconstrained.}
    \label{fig:dapo_loss}
\end{figure}

\newpage
\begin{figure}[h]
    \centering
    \includegraphics[width=0.9\linewidth]{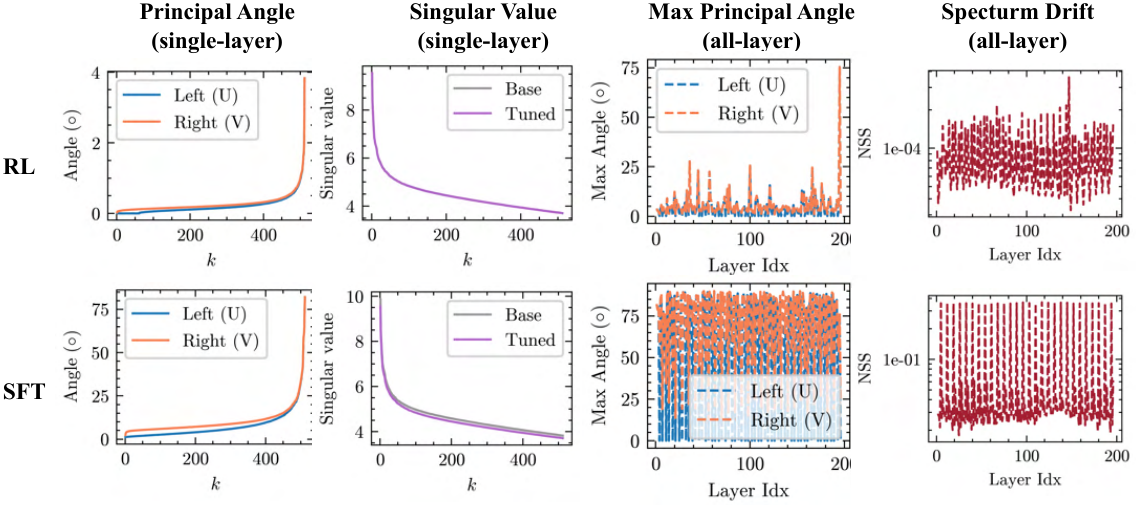}
    \vspace{-10pt}
    \caption{The spectrum probe results on the RL and SFT version on the \texttt{DS-Distill-Qwen-1.5B}~\cite{liu2025prorl}. RLVR shows surprisingly stable top-k spectrum with minimal subspace rotation and top-k eigenvalue changes.
    }
    \label{fig:rl-spec-1.5b}
    \vspace{-10pt}
\end{figure}

\begin{figure}[h]
    \centering
    \includegraphics[width=0.9\linewidth]{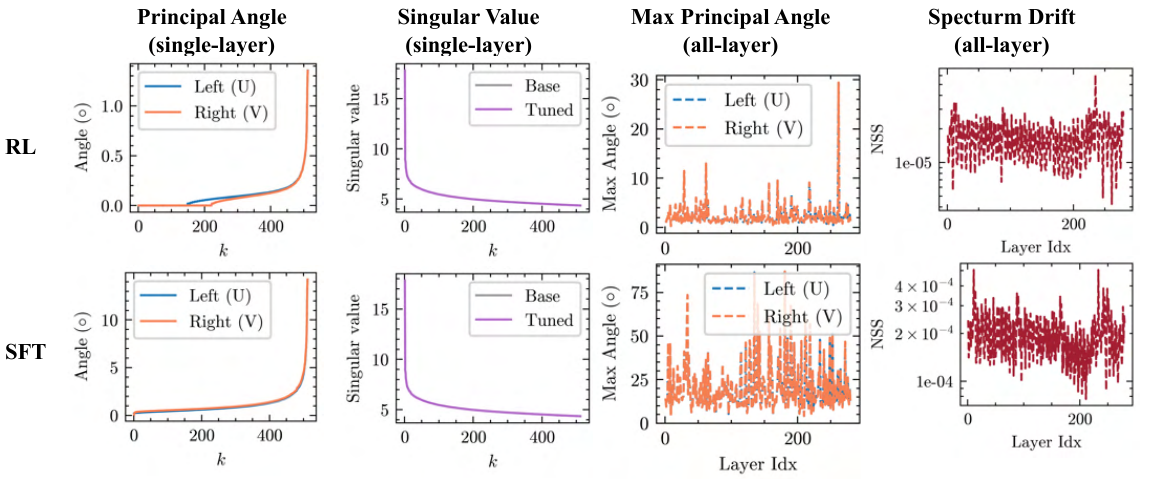}
    \vspace{-10pt}
    \caption{The spectrum probe results on the RL and SFT version on the Qwen3-14B~\cite{huan2025does}. RLVR shows surprisingly stable top-k spectrum with minimal subspace rotation and top-k eigenvalue changes.
    }
    \label{fig:rl-spec-14b}
    \vspace{-10pt}
\end{figure}

\begin{figure}
    \centering
    \includegraphics[width=0.9\linewidth]{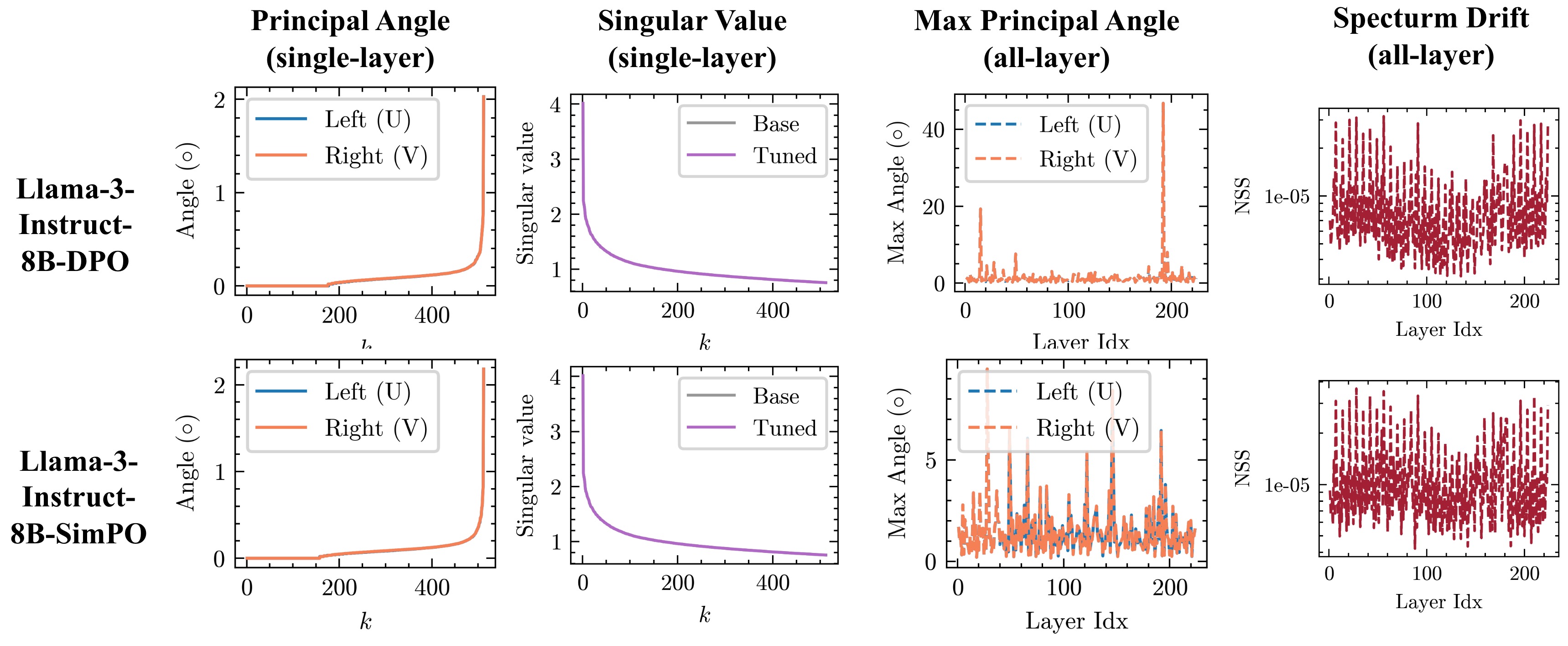}
    \vspace{-10pt}
    \caption{\small \textbf{Spectral geometry under RLHF setting~\cite{meng2024simpo}.}
Across RLHF checkpoints, RL training preserves layer spectra and induces only minor rotation of the top-$k$ subspaces, consistent with the RLVR regime.}
    \label{fig:rl-spec-rlhf}
    \vspace{-10pt}
\end{figure}

\end{document}